%% file: main.tex
\documentclass{article}

 \usepackage[preprint]{neurips_2026}

\usepackage[utf8]{inputenc} 
\usepackage[T1]{fontenc}    
\usepackage{hyperref}       
\usepackage{url}            
\usepackage{booktabs}       
\usepackage{amsfonts}       
\usepackage{nicefrac}       
\usepackage{microtype}      
\usepackage{xcolor}         
\usepackage{amsmath}
\usepackage{amssymb}
\usepackage{mathtools}
\usepackage{amsthm}
\usepackage{microtype}
\usepackage{graphicx}
\usepackage{subcaption}
\usepackage{booktabs} 

\usepackage{algorithm}
\usepackage{algorithmic}

\usepackage[capitalize,noabbrev]{cleveref}

\theoremstyle{plain}
\newtheorem{theorem}{Theorem}[section]
\newtheorem{proposition}[theorem]{Proposition}
\newtheorem{lemma}[theorem]{Lemma}

\theoremstyle{definition}
\newtheorem{definition}[theorem]{Definition}
\newtheorem{assumption}[theorem]{Assumption}
\theoremstyle{remark}
\newtheorem{remark}[theorem]{Remark}

\usepackage[textsize=tiny]{todonotes}
\usepackage{mystyle}

\newcommand{\R}{\mathbb{R}}
\newcommand{\1}{\mathbf{1}}

\title{Explicit Credit Assignment through Local Rewards and Dependence Graphs}

\author{%
  Bang Giang Le$^{1}$ \quad Viet Cuong Ta$^{1}$\thanks{Corresponding author.}\\
  $^1$Human-Machine Interaction Laboratory\\
  VNU University of Engineering and Technology, Hanoi, Vietnam\\
  \texttt{\{giangbang, cuongtv\}@vnu.edu.vn} \\
}

\begin{document}

\maketitle

\begin{abstract}
\input{abstract}
\end{abstract}

\input{intro}

\input{setting}

\input{motivation}

\input{method}

\input{experiment}

\input{conclusion}



\bibliographystyle{achemso}
\bibliography{reference}


\appendix

\input{appendix}

\clearpage

\end{document}

%% file: abstract.tex
To promote cooperation in Multi-Agent Reinforcement Learning, the reward signals of all agents can be aggregated together, forming global rewards that are commonly known as the fully cooperative setting. 
However, global rewards are usually noisy because they contain the contributions of all agents, which have to be resolved in the credit assignment process. On the other hand, using local reward benefits from faster learning due to the separation of agents' contributions, but can be suboptimal as agents myopically optimize their own reward while disregarding the global optimality. 
In this work, we propose a method that combines the merits of both approaches.
By using a dependence graph between agents' interaction, our method discerns the individual agent contribution in a more fine-grained manner than a global reward, while alleviating the cooperation problem with agents' local reward. 
We also introduce a practical approach for approximating such a graph.
We test the proposed approach with known dependence graphs and learned dependence graphs in fully cooperative settings.
The results demonstrate the flexibility of the approach, enabling improvements over the traditional local and global reward settings.

%% file: intro.tex
\section{Introduction}

Cooperation between multiple agents in Multi-Agent Reinforcement Learning (MARL) can be understood as the collective effort whereby multiple agents jointly pursue a shared objective, thereby achieving a better outcome that might be unattainable by individual agents alone. 
Such collaboration is crucial in real-world applications, where tasks such as autonomous driving \cite{shalev2016safe}, distributed voltage control \cite{wang2021multi}, or swarm coordination \cite{cao2012overview} are often too complex for isolated agents to solve effectively.
As a result, Cooperative MARL remains a central research direction within multi-agent learning \cite{yuan2023survey}. 

The simplest approach to enforce cooperation within a group of agents is through reward aggregation (scalarization), where the rewards of individuals are replaced by the sum (average) of the agents' rewards in the whole group \cite{samvelyan2019starcraft, lowe2017multi}. 
In this way, all agents share a common objective, which is the reward of the whole group; the loss of the group is the loss of every agent. This setting generally goes under the name \textit{fully cooperative} settings in the literature \cite{yu2022surprising, rashid2020monotonic}.
Rewards scalarization, however, introduces the problem of \textit{credit assignment} \cite{chang2003all, foerster2018counterfactual}, where the information on the contributions of each agent is lost through the scalarization operations. 

Alternatively, in a general \textit{cooperative} setting, individual rewards are kept separate, and each agent optimizes its own objective independently, which effectively reduces the Multi-agent problem to a set of single-agent RL tasks \cite{de2020independent}. 
This learning approach is often referred to as \textit{independent learning}, with examples including IQL \cite{tan1993multi}, IPPO \cite{de2020independent}, and MADDPG \cite{lowe2017multi}. In this approach, each agent optimizes its own reward signal myopically, and therefore, agents only cooperate when cooperation is beneficial to them.
In this case, cooperation can only emerge implicitly \cite{zheng2018magent}. 
Learning from local rewards can potentially lead to selfish behaviors that can be suboptimal to the entire system \cite{devlin2014potential, le2025toward}.

In fully cooperative MARL, it is often assumed that all agents share a single global reward.
To mitigate the resulting credit assignment problem, various methods have been proposed, such as value decomposition \cite{sunehag2017value} and counterfactual learning \cite{foerster2018counterfactual}.
However, in many scenarios, the team’s reward can be naturally attributed to the contributions of specific subsets of agents, in which case local rewards can provide more fine-grained information about overall performance.
Consequently, aggregating these signals into a single scalar and then attempting to recover them appears inherently inefficient. 
Motivated by this observation, we ask 
\begin{quote}
\textit{Can we directly leverage the individual-agent reward information to effectively avoid the credit assignment issues introduced by reward scalarization, while still promoting cooperation?}
\end{quote}

Building on the framework of dependence graphs and Multi-Agent Networked MDPs \citep{qu2020scalable}, we combine local reward signals in a way that respects the credit of each agent.
Our main contributions in this work are three-fold. 
First, we propose a novel policy gradient estimator that truncates irrelevant local rewards based on interaction paths in the dependence graph, effectively alleviating the credit assignment problem. 
Second, we prove that our method provides a smooth transition between the local and global reward learning depending on the density of the dependence graphs in terms of sample complexity.
Third, we introduce a practical method for approximating the dependence graph based on our theory.
Empirically, we validate the proposed approach on a range of benchmarks in MARL. Experimental results show that our methods can be competitive with and outperform strong baselines from both local and global reward settings.

%% file: setting.tex
\section{Problem Setting and Notation}\label{sec:settings}

\textbf{Markov Game.} We formulate the MARL problem as a Markov game, defined by the tuple $\mathcal M =\langle \mathcal{N}, \boldsymbol{\mathcal{S}}, \boldsymbol{\mathcal{A}}, \mathbf P, \mathbf r, \gamma \rangle$. Here, $\mathcal{N} = \{1, 2, \dots, N\}$ denotes the set of $N$ agents, $\boldsymbol{\mathcal{S}}$ is the state space, and $\boldsymbol{\mathcal{A}} = \mathcal{A}_1 \times \dots \times \mathcal{A}_N$ represents the joint action space, where $\mathcal{A}_i$ is the action space of agent $i$. The function $\mathbf P$ defines the transition dynamics, $\mathbf r:\boldsymbol{\mathcal S\times \boldsymbol{\mathcal A}}\to [0, R_{\max}]^N$ is the joint reward function, and $\gamma \in [0, 1)$ is the discount factor.
Throughout the paper, we use bold notation to indicate joint quantities: for example, the joint action at time step $t$ is denoted by $\boldsymbol{a}_t = (a_t^1, \dots, a_t^N)$, and the joint policy by $\boldsymbol{\pi} = (\pi_1, \dots, \pi_{N})$. At each time step $t$, all agents \textit{independently} select actions according to their policies, forming the joint action $\boldsymbol{a}_t$, upon which the environment transitions to the next state $s_{t+1}$ according to the transition kernel $\mathbf P$.
Additionally, let $\rho_{\boldsymbol {\pi}} (\mathbf s) = (1-\gamma)\sum_{t=0}^\infty \gamma^t \mathbb P_{\boldsymbol{\pi}}(\mathbf {s}_t = \mathbf s)$ be the discounted marginal state distribution of policy $\boldsymbol{\pi}$.

\textbf{Networked Multi-Agent MDP.}
Many practical MARL problems exhibit local dependence structures, where agents may interact with a small subset of neighbors and are relatively independent of other outsider agents. 
To facilitate such structural dependency, we utilize the framework of networked Markov systems \citep{bagnell2005local, qu2019exploiting, qu2020scalable}. 
More specifically, we assume that the joint state space $\boldsymbol{\mathcal{S}}$ can be decomposed into substate spaces of individual agents, i.e., $\boldsymbol{\mathcal S}=\mathcal S_1\times \dots \times \mathcal S_N$,
and an agent can only observe their respective local states, $\pi_i: {\mathcal  {S}}_i\to \Delta(\mathcal A_i)$.
The transition kernel $\mathbf P$ also factorizes accordingly: the
next state of agent $i$ depends only on a local agent-dependence set, denoted by $\text{Pa}^i_{}(s^i) \subseteq  [N]$, with
\[
P_i(s^{i}{'}|\mathbf s, \mathbf a)=P_i \bigl( s^{i}{'} \,\big|\,  \{s^k, a^k : k \in \text{Pa}^i_{}(s^i) \}\bigr)
\in \Delta(\mathcal{S}_i).
\]
Note that the agent dependence set depends on the local states; agents can infer from their observations which other agents can affect their next state.
We assume $i \in \text{Pa}_{}^i( s^i)$, that is, each agent always influences its own next state. 
This decomposition structure naturally forms a \emph{(state-)dependence graph}, which we define later in this section.

Each agent $i$ also receives an individual reward
$r^i : \mathcal{S}_i \times \mathcal{A}_i \to [0, R_{\max}]$ based on their states and actions.
For clarity of exposition in the main text, we restrict to the
case where $r^i$ depends only on the local state-action
pair $(s^i,a^i)$. More general settings of the reward functions can be naturally captured by introducing a second \emph{reward-dependence graph}. 
The reward-dependence graph captures myopic interactions: the immediate effect
to current rewards, future propagation is already captured by the state-dependence graph. Since this extension does not alter the main theory and can be incorporated with minor modifications, we defer its full formalization to Appendix \ref{app:reward-graph}.

We define the cumulative returns of agent~$i$ under the joint policy~$\boldsymbol{\pi}$ as \(J^i(\boldsymbol{\pi})= \mathbb E_{\boldsymbol{\pi}}\!\left[\sum_{t=0}^{\infty} \gamma^t r_t^i \right].\)
In the fully cooperative setting, agents optimize a shared objective by maximizing the team’s total expected return,
\(J(\boldsymbol{\pi})= \sum_{i=1}^N J^i(\boldsymbol{\pi}),\) with $r_t = \sum_{i=1}^N r_t^i$ the scalarized reward used as a common learning signal for all agents.
Throughout the paper, we refer to the scalarized reward $r_t$ as \textit{global reward} to distinguish it from the \emph{local reward} $r^i_t$.  


The factorization of the global states resembles the partial observability property, where each agent only observes parts of the underlying states. However, unlike partial observation, the transition kernel directly operates on the partial states.
Furthermore, one can recover the standard formulation by duplicating the true states, $\mathbf s = (s, \dots, s)$, and by defining the transition kernel $P_i$ to include all the agents in the set $\text {Pa}^i_{\cdot}$. 
Our formulation differs from prior work in the underlying dependence structure. In \citet{qu2019exploiting}, the dependencies form a tree, which corresponds to a dependence set of size one in our framework. Other works \citep{qu2020scalable, jing2024distributed} consider a graph-structured formulation, but the graph is fixed, whereas in our setting, the dependence graph is dynamic and can vary across states, which admits static graphs as special cases.


\input{graph_defined}

\textbf{Information-Theoretic Notation.} We use capital letters to denote random variables corresponding to the same entities; for example, $S^i$ and $A^i$ represent the random variables associated with the (sub)state $s^i$ and action $a^i$ of the agent $i$. We also use $H(\cdot)$ to denote entropy and $I(\cdot; \cdot)$ to denote mutual information.

%% file: graph_defined.tex
\textbf{State Dependence Graph.}
Based on the Dependence structure, we define the {dependence graph}, in which a directed edge exists between two agents if one agent influences the other's next states.

\begin{definition}[State Dependence Graph induced by a Networked Multi-Agent MDP]
Let $\mathcal M$ be a Networked Multi-Agent MDP.
The \emph{State Dependence Graph} of $\mathcal M$ is a directed graph
$\mathcal G = \langle \mathcal V, \mathcal E\rangle$, where $\mathcal V = \{ (\mathbf s, i) \mid \mathbf s \in \mathcal S,\; i\in [N] \}$ is the vertex set and $\mathcal E$ is the edge set. 
A directed edge
\(
(\mathbf s, i) \to (\mathbf s', j)
\)
belongs to $\mathcal E$ if $i \in \mathrm{Pa}^j_{}(s^j)$ and $\exists \mathbf a \in \boldsymbol{\mathcal A}$ such that $  \mathbf P(\mathbf s'|\mathbf s, \mathbf a) > 0$.
\end{definition}

\begin{definition}[Proper Dependence Graph]
A graph $\mathcal G=\langle \mathcal V, \mathcal E\rangle$ is proper if there exists a transition kernel $\mathbf P$ that admits it as a dependence graph; i.e. if $(\mathbf s_1, i)\to(\mathbf s_1', j)$ and $(\mathbf s_2, i)\to(\mathbf s_2', i)$ are two edges of $\mathcal G$ such that $ s_1^j= s_2^j$, then $(\mathbf s_2, i) \to (\mathbf s_2', j)$ is also an edge of $\mathcal G$. 
Every proper dependence graph has a unique collection of agent-dependence sets $\{\text{Pa}_{\mathcal G}^i(s^i)\}_{i}^N$ for each state $\mathbf s$.
\end{definition}


Fix a realized state trajectory $\mathbf s_{0:\infty}=(\mathbf s_0,\mathbf s_1,\dots)$, the $k$-hop parent sets of agent $i$ at timestep $t$ are defined recursively as:
\[
\text{Pa}^i(s_t^i,0,\mathbf s_{0:t})=\{i\},
\qquad
\text{Pa}^i(s_t^i,k,\mathbf s_{0:t}))=\bigcup_{m\in\text{Pa}_i(s_t^i,k-1,\mathbf s_{0:t})}\text{Pa}^m(s_{t-k}^m),
\quad k\ge 1.
\] 
The \emph{first meeting time} between agent $j$ starting from timestep $t$ to the agent $i$ in a trajectory $\mathbf s_{0:\infty}$ in a future timestep $t'>t$ is defined as
\begin{equation}\label{def:meetingtime}
T_{ji}(s^j_t, \mathbf s_{0:\infty}, \mathcal G):=\min\{t'\ge t:j\in\text{Pa}^i(s_{t'}^i,t'-t,\mathbf s_{0:t'})\}-t,
\end{equation}
with $T_{ij}(\cdot)=\infty$ if no such $t'$ exists. We also write $T_{ji}( s)$ if the dependence graph and the trajectory are obvious from the context. The meeting time can be interpreted as the rate at which interactions between agents propagate across the graph.

%% file: motivation.tex
\section{Motivation: Scalarize or not scalarize}




Existing works in cooperative MARL either use a global reward, for example, in CTDE, or a local reward, such as in independent learning (see Section \ref{sec:relate}). However, these two setups represent two extreme learning paradigms that are only suitable on a case-by-case basis. 
To illustrate, we provide an example of these two learning modes in the LBF environments \citep{christianos2020shared}. 
We demonstrate two scenarios that exhibit contrasting learning patterns on the same backbone MAPPO algorithm. The results are presented in Figure (\ref{fig:motivation}).
We observe the following;

\begin{wrapfigure}[17]{r}{0.49\textwidth}
\centering
\includegraphics[width=0.49\textwidth]{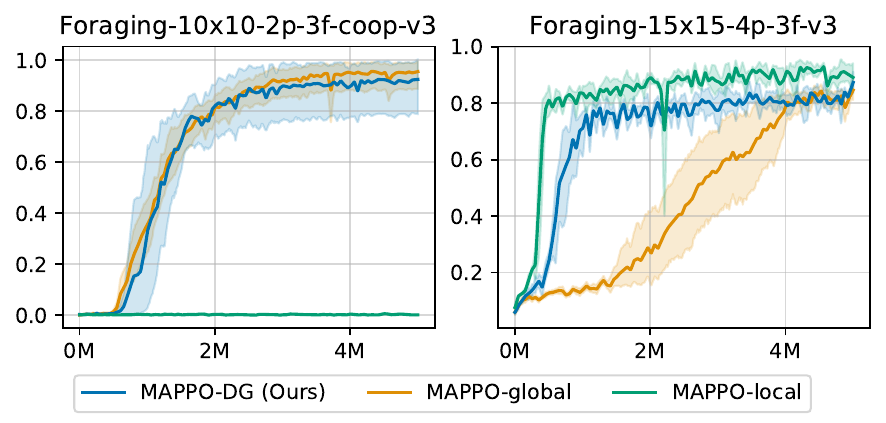}
\caption{Reward dilemma in cooperative MARL; global reward enhances cooperation but introduces credit assignment problem, while local reward can induce suboptimal policies in the environments that require cooperation. Our method can enable faster training, as in local rewards, while avoiding the miscoordination pitfall.} \label{fig:motivation}
\end{wrapfigure}
On the one hand, learning with a global reward signal can ensure the cooperation of agents. Among the two scenarios, the global reward settings both converge. However, the variance of the estimated gradient in environments with a high number of agents can be significantly high; this additional variance is the cost of the credit assignment from scalarization. 
In fact, any global-reward algorithms face a theoretical sample-efficiency limit can scale with the number of agents \cite{bagnell2005local}, exacerbating the credit assignment problems \citep{foerster2018counterfactual, kuba2021settling}.
Existing works alleviate this problem either through adaptive trust region \citep{sun2022trust}, sequential updates \citep{zhong2024heterogeneous}, or conservative update mechanisms \citep{wu2021coordinated}.

On the other hand, learning with separate reward signals can converge faster, especially in environments with a large number of agents, since local reward learning does not suffer from the credit assignment problem.
However, the cost of this learning approach is the potential of suboptimal convergence \cite{devlin2014potential, le2025toward}.
Resolving this problem generally requires careful reward engineering and experimentations \citep{mao2020reward}.

We summarize our observation as the following \textit{reward dilemma} for multi-agent RL
\begin{textbox}
    The credit assignment problem introduced by global reward scalarization imposes additional variance that scales with the number of agents, potentially slowing down the learning process. However, local reward learning can cause severe consequences due to miscoordination.
\end{textbox}

From this observation, we see the limitations of both the global scalarization and local reward learning. 
Restricting learning to only one of these settings can be insufficient for a truly general learning framework. 
This motivates the development of a framework that supports more flexible learning paradigms, allowing a smooth transition between these two extremes. 
In this work, we take a step in this direction.
Our approach can offer greater flexibility to reward settings and facilitate more effective optimization in MARL, while mitigating the credit assignment issues of global rewards and the miscoordination in local rewards. 

%% file: method.tex
\section{Method}\label{sec:method}

\subsection{Dependence Graph Policy Gradient}
\label{sec:dgpg}
We now derive how the dependence structure introduced in the previous section enables a more efficient computation of the policy gradient. Our starting point is the fully cooperative objective,
\(J(\boldsymbol{\pi}) = \sum_{i=1}^N J^i(\boldsymbol{\pi})\).
The policy gradient with respect to agent \(j\)'s parameters can then be written as $\nabla_{\pi_j} J(\boldsymbol{\pi}) = \sum_{i=1}^N \nabla_{\pi_j} J^i(\boldsymbol{\pi})$.
Using the multi-agent policy gradient theorem \citep{foerster2018counterfactual, kuba2021settling}, we have
\begin{equation}
    \nabla_{\pi_j} J^i(\boldsymbol{\pi})
    = 
    \frac{1}{(1-\gamma)}\mathbb E_{\rho_{\boldsymbol{\pi}},\, \boldsymbol{\pi}}
    \Bigl[
        \nabla \log \pi_j(a^j|s^j)\,
        Q^i_{\boldsymbol{\pi}}(\mathbf s, \mathbf a)
    \Bigr],
    \label{eq:cross-grad}
\end{equation}
where \( Q^i_{\boldsymbol{\pi}}(\mathbf s, \mathbf a) \) denotes the joint action-value function of agent \(i\) under the joint policy \( \boldsymbol{\pi} \).
We include the derivation of \eqref{eq:cross-grad} in appendix \ref{sec:multi_agent_pg} for completeness.
In local reward settings, the cross-agent gradients are completely ignored; i.e. $\nabla_{\pi_i} J^j=0, \forall j \neq i$.

We can further simplify the policy gradient in (\ref{eq:cross-grad}) to include only the necessary states that affect agent $i$'s rewards.
As illustrated in Figure \ref{fig:mdp}, 
if agent $k$ takes an action at timestep $t$, it cannot directly affect the state of agent $p$ at timestep $t+1$
as there is no edge connecting them.
Therefore, the cross-gradient of agent $k$ to $p$ is zero in this case.
However, the state of agent $q$ at time $t+2$ can be influenced by agent $k$, but only through an intermediate state, which is an indirect effect that "propagates" from agent $k$ through paths on the graph.
We formalize this propagation effect in the following result.
\begin{proposition}\label{prop:global}
    Fix a joint policy $\boldsymbol{\pi}$. 
    Let $i, j \in  [N]$.
    The policy gradient $\nabla_{\pi_j}J^i(\boldsymbol{\pi})$ is given by
    \begin{equation}
        \frac{1}{1-\gamma}\mathbb E_{\mathbf s_0, \mathbf a_0\sim\rho_{\boldsymbol{\pi}}, \boldsymbol{\pi}}\left[\nabla \log \pi_j(a_0^j|s_0^j)\mathbb E_{\tau}\bigg [ \sum_{k \geq T_{ji}( s^j_0, \tau)}^\infty\gamma^{k}r^i(s^i_{k}, a^i_{k})\bigg| \mathbf s_0, \mathbf a_0\bigg]\right], \label{eq:dependency_graph_pg}
    \end{equation}
    where $T_{ji}(s^j_0, \tau)$ is the first timestep first meeting time as defined in \eqref{def:meetingtime}.
\end{proposition}
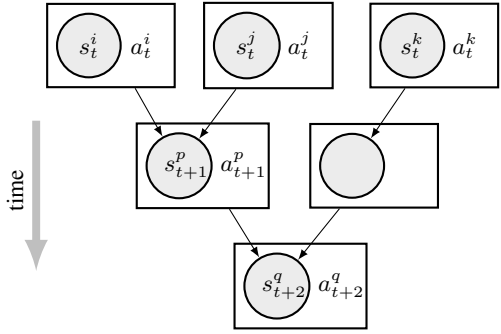
\begin{wrapfigure}[21]{r}{0.49\textwidth}
\input{mdp_fig}
\caption{An example of an MDP with decomposed state structure. Agent $k$ (top right) cannot influence agent $q$ at timestep $t+1$; any effect can only occur from timestep $t+2$ onward (bottom). Since agent $q$ lies in agent $k$'s blind spot, it can be excluded from the gradient computation at $t+1$. In this example, the meeting time $T_{kq}(s^k_t)=2$.}\label{fig:mdp}
\end{wrapfigure}
We call the policy gradient that is estimated through Proposition \ref{prop:global} the \textit{dependency-graph} policy gradient, and denote it by $\nabla_{\pi_j}J(\boldsymbol{\pi}; \mathcal G)$. 
Compared to the traditional policy gradient, equation \eqref{eq:dependency_graph_pg} differs by the cumulative reward that only starts after the timestep $T_{ji}$. As a result, if each agent operates relatively independently, i.e., a sparse dependence graph, then the cross-agent gradient can become exponentially small based on the meeting times. 
On the contrary, if all agents' MDPs are tightly coupled, e.g., a fully connected dependence graph, then \eqref{eq:dependency_graph_pg} returns to the standard policy gradient.

Our dependency-graph Policy Gradient operates on an orthogonal dimension and is complementary to the Policy gradient with Networked MDPs in \citet{qu2019exploiting, qu2020scalable}. 
More specifically, compared to Lemma 2 in \cite{qu2020scalable}, where the joint value functions are approximated through a small subset of $k$-hop parents, with the approximation error reducing exponentially fast with $k$, i.e., the value function \textit{representation} dimension. Our policy gradient operates on the \textit{time} dimension where the value functions can be estimated after the meeting timestep, without changing the gradient.
Moreover, compared with the policy gradient theorems in \cite{syed2026structured} (Theorem 3.7) and \cite{jing2024distributed} (Lemma 2), our result generalizes the pruning scheme.
Specifically, their theorems only prune agents that are unreachable from a given agent, corresponding to the special case $T_{ji}=\infty$ in our framework. As a result, their pruning mechanism is effective only when the graphs have multiple connected components.

\input{sample_analysis}

\input{buildgraph}

\input{analysis}

%% file: mdp_fig.tex
\begin{tikzpicture}[auto,node distance=2mm,>=latex,font=\small,
  text width=4mm,align=center,]
\tikzstyle{round}=[thick,draw=black,circle,fill=gray!15]
\tikzstyle{square}=[thick,draw=black,rectangle, minimum size=1mm]



\node[round] (sti) at (-0.3,0) {$s_t^i$};
\node (ati) at (0.4,0) {$a_t^i$};

\node[square, fit=(sti)(ati), name=ti,minimum height=1cm] {};

\node[round] (sti1) at (1.8,0) {$s_t^j$};
\node (ati1) at (2.5,0) {$a_t^j$};

\node[square, fit=(sti1)(ati1), name=ti1,minimum height=1cm] {};

\node[round] (sti3) at (4.0,0) {$s_t^k$};
\node (ati3) at (4.7,0) {$a_t^k$};

\node[square, fit=(sti3)(ati3), name=ti3,minimum height=1cm] {};


\node[round] (st1i) at (0.9,-1.6) {$s_{t+1}^p$};
\node (at1i) at (1.65,-1.6) {$a_{t+1}^p$};

\node[square, fit=(st1i)(at1i), name=t1i1,minimum height=1cm] {};

\node[round, minimum size=8.5mm] (st1i1) at (3.2,-1.6) { };
\node (at1i1) at (3.9,-1.6) {};

\node[square, fit=(st1i1)(at1i1), name=t1i2,minimum height=1cm] {};


\node[round] (st2i) at (2.2,-3.2) {$s_{t+2}^q$};
\node (at2i) at (2.95,-3.2) {$a_{t+2}^q$};

\node[square, fit=(st2i)(at2i), name=t2i1,minimum height=1cm] {};

\draw[->] (ti) -- (st1i);
\draw[->] (ti1) -- (st1i);
\draw[->] (t1i1) -- (st2i);

\draw[->] (ti3) -- (st1i1);
\draw[->] (t1i2) -- (st2i);


  \draw[->, draw=gray!50, line width=3pt] (-1,-1) -- (-1,-3.0) node[midway, left=8pt, rotate=90, anchor=center] {time};

\end{tikzpicture}

%% file: sample_analysis.tex
\subsection{Sample Complexity Analysis} \label{sec:sample_complexity}
In this section, we demonstrate sample complexity improvements of using the dependence graph through Policy Gradient (PG) algorithms. We let $\boldsymbol{\theta}=(\theta_1, \theta_2, \dots, \theta_N)$ to denote the joint policy parameters, where each $\theta_i$ parameterises agent $i$'s policy. 
In the policy gradient framework, we estimate gradients from $m$ sampled trajectories over a truncated horizon $H$.
The standard estimator and its dependence-graph counterpart are given by:
\begin{equation}
\hat{\nabla}^{m}_{\theta_j}J^i(\boldsymbol{\theta}) = \frac{1}{m}\sum_{l=1}^m\sum_{t=0}^{H}\gamma^t \nabla_{\theta_j}\log \pi_j(a^i_{t, l}|s^i_{t, l}) \sum_{k=t}^H\gamma^k r_{k, l}^i \label{eq:gpomdp}
\end{equation}
\begin{equation}\hat{\nabla}^m_{\theta_j}J^i(\boldsymbol{\theta}, \mathcal G) = \frac{1}{m}\sum_{l=1}^m\sum_{t=0}^{H}\gamma^t \nabla_{\theta_j}\log \pi_j(a^j_{t, l}|s^j_{t, l}) \sum_{k=t + T_{ji}^H(s_{t, l}^j, \mathbf s_{0:H, l})}^H\gamma^k r_{k, l}^i\label{eq:grad_graph}
\end{equation}
With truncated dependence-graph gradient estimates, cross-agent gradients are dropped beyond horizon $H$ \footnote{We define $T_{ji}^H$ in \eqref{eq:truncated_graph_gradient_estimate}.}. 
Following gradient estimation, each agent $j$ updates its parameters as
$\theta_j^{n} = \theta_j^{n-1} + \eta \sum_i^N\hat{\nabla}_{\theta_j}J^i(\cdot), \forall j \in [N]$, where $\eta$ is the step size and $\theta^n_j$ is the agents' parameters at the iteration $n$. This process repeats for a total of $T$ iterations.  Also, we note that the update in \eqref{eq:gpomdp} corresponds to the GPOMDP estimator in the single-agent literature \citep{baxter2001infinite}. 

Compared to the standard PG, the dependence-graph estimator \eqref{eq:grad_graph} prunes timesteps prior to $T_{ji}$, yielding a reduction in effective sample weight by a factor of approximately $\gamma^{T_{ji}}$. To formalise this gain, we define
\[\Gamma_j^i := \sup_{k, \boldsymbol{\pi}} \mathbb E[\gamma^{T_{ji}(s_k^j)}]\leq 1 \quad \quad \Gamma_j:=\sum_i^N \Gamma_j^i \leq N \quad \quad \Gamma :=\sum_j^N \Gamma_j \leq N^2, \]
The following result shows that exploiting the dependence structure leads directly to improved sample complexity.
\begin{theorem}\label{theom:sample_complexity}
    Suppose that each agent's policy satisfies the standard smoothness and Lipschitz condition (assumption \ref{assm:e-ls}). Then, for a given $\epsilon > 0$, by choosing appropriate parameters of $H, T, m, \eta$ (details in appendix \ref{sec:main_proof}), we need a total number of samples of 
    \begin{equation*}
        \begin{cases}
        {O}\left( \frac{NR_{\max}^4\Gamma\max_{j}\Gamma_j \log (N/\epsilon)}{(1-\gamma)^6\epsilon^4\log(1/\gamma) }\right) & \text{ using dependence graph PG \eqref{eq:grad_graph}}\\
        {O}\left( \frac{N^3R_{\max}^4 \max_j\Gamma_j \log (N/\epsilon)}{(1-\gamma)^6\epsilon^4\log(1/\gamma) }\right) & \text{ using standard PG \eqref{eq:gpomdp}}
    \end{cases}
    \end{equation*}
    to guarantee convergence to an $\epsilon$ stationary point $\mathbb E[\|\nabla_{\boldsymbol{\theta}^U}J(\boldsymbol{\theta}^U)\|^2]\leq O(\epsilon^2)$, where $\boldsymbol{\theta}^U$ is uniformly sampled from $\{\boldsymbol{\theta}^0, \dots, \boldsymbol{\theta}^{T-1}\}$. 
    Furthermore, if the gradient domination condition holds for each agent (assumption \ref{assm:gradient_domination}), then the converged point is an $\epsilon$ equilibrium.
    \[ \mathbb E[J(\boldsymbol{\boldsymbol{\theta}}^U)]\geq  \sup_{\pi_j}\mathbb E [J(\boldsymbol{\pi}_{\boldsymbol{\theta}^U}^{-j}, \pi^j)]-O(\epsilon), \quad \forall j\in [N].\]
\end{theorem}
\begin{remark}
    The term $\max_j\Gamma_j$ appears in both complexity bounds because the smoothness of the objective function $J(\boldsymbol{\theta})$ (the hardness of the optimisation problem) scales according to the dependence structure of the MDP. The standard PG estimates incur a complexity that scales with $N^2$ as it cannot exploit the local reward structure.  
    On the other hand, the dependence-graph estimator improves this scale with $\Gamma$ instead, which can be substantially smaller when inter-agent dependencies are sparse.
\end{remark}
We compare our result with the state-of-the-art sample complexity of single-agent PG methods \citep{yuan2022general} in the following two special cases.
\begin{remark}\label{rem:sparse_graph}
    As a first special case, when the dependence graph is totally disconnected (no edges exist between vertices of different agents, i.e. no cross-agents dependency or interactions), then our MDP formulation behaves as if a vector of parallel (and different) single-agent MDPs, where each "sample" in this vector MDP comprises $N$ "subsample" tuples. In this case, we have that $\Gamma=O(N)$ and $\max_j \Gamma_j=O(1)$, and so the sample complexity of the dependence graph scales with $N^2$ vector samples, which becomes $N^3$ if we consider subsamples. 
    As a result, if we consider a subsample in the vector samples equivalent to a sample in single-agent problems, then the sample complexity of our approach matches (up to a logarithmic factor) the sample complexity of solving $N$ independent single-agent RL problems with the PG method to an equivalent $\epsilon/\sqrt{N}$ precision.
\end{remark}
\begin{remark} \label{rem:dense_graph}
    In the second special case, when the dependence graph is fully connected, we have that $\Gamma=O(N^2)$ and $\max_j \Gamma_j=O(N)$, and so the sample complexity of the dependence graph PG scales with $N^4$. In this case, the sample complexity of our method matches the sample complexity of solving a single agent RL in the joint state action space (upto a constant factor due to the mismatch of the smoothness and Lipchitzness coefficients in the joint and local parameter space) using PG method with the new $R^{\text{new}}_{\max}=NR_{\max}$.
\end{remark}
Remark \ref{rem:sparse_graph} and \ref{rem:dense_graph} demonstrate that our method provides a smooth transition between the two extreme settings of global and local learning depending on the sparsity of the agent interactions.

%% file: buildgraph.tex
\subsection{Dependence Graph Policy Gradient with approximate Graphs}

The formulation of the dependence graph $\mathcal G$ introduced earlier assumes knowledge of the underlying transition dynamics $\mathbf P$.  
In practice, however, the transition model is often unknown, necessitating the use of approximate or heuristic graphs.
In fact, prior works either assume that the graph is given \citep{qu2020scalable}, or use predefined structures \citep{ma2024efficient, jing2024distributed}.
In this section, we consider how an approximate graph $\mathcal G'$ can be used in place of the true graph. We first establish a condition under which the estimate policy gradient remains accurate.  

\begin{lemma}[Gradient Error under an Approximated Dependence Graph]\label{lem:graph-approx}
Let $\mathcal G$ be the true dependence graph induced by the transition kernel $\mathbf P$, and let $\mathcal G'$ be a proper approximate graph for which there exists a transition kernel $\mathbf P'_{\mathcal G'}$ that admits $\mathcal G'$ as its dependence graph and satisfies
\begin{equation}
    \sup_{\mathbf s,\mathbf a}\| \mathbf P(\cdot\mid\mathbf s,\mathbf a) - \mathbf P{'}_{\mathcal G'}(\cdot\mid\mathbf s,\mathbf a)\|_1 \le \varepsilon. \label{eq:lem_approx}
\end{equation}
Let  $\nabla_{\pi_j}J(\boldsymbol\pi;\mathcal G')$ denote the dependency-graph policy gradient from the graph $\mathcal G'$. 
If rewards are bounded in $[0,R_{\max}]$ and
\(
B_j := \sup_{s^j,a^j}\|\nabla\log\pi_j(a^j\mid s^j)\|
\)
is finite, then
\[
\big\|
\nabla_{\pi_j}J^i(\boldsymbol\pi) - \nabla_{\pi_j}J^i(\boldsymbol\pi;\mathcal G')
\big\|
\le
O\left( B_j \,\frac{\gamma\,\varepsilon\,R_{\max}}{(1-\gamma)^2}\right).
\]
\end{lemma}

%% file: analysis.tex
The subtlety of lemma \ref{lem:graph-approx} is the use of the approximated graph $\mathcal G'$ instead of $\mathcal G$ in the policy gradient while using samples from $\mathbf P$. The existence of the kernel $\mathbf P'_{\mathcal G'}$ is to only guarantee that $\mathcal G'$ represents something sensible but is not used otherwise. 
While Lemma \ref{lem:graph-approx} can be applied to any approximated graphs, including the heuristic ones, it also suggests a theoretical approach for estimating such a dependence graph through a model-based approach; by learning a dynamic model $\mathbf P'$, then the induced graph $\mathcal G'$ gives small-bias estimation of the gradient as long as the model prediction error is small. 
However, estimating the total variation in (\ref{eq:lem_approx}) of the world model can be challenging because it usually involves learning a high-dimensional, complex next-state distribution. 
Furthermore, it is unclear how to extract a dependence graph from a learned world model in practice.
We next consider this quantity in a slightly different form.

For a given proper graph $\mathcal G'$, we let $\mathbf P_{\mathcal G'}$ be the probability of next states conditional on the graph structure of $\mathcal G'$, i.e.
\[ P^i_{\mathcal G'}(s^i{'}| \mathbf s, \mathbf a) =   P^i\big(s^i{'} \big| \{s^k,  a^k :k\in \text{Pa}^i_{\mathcal G'}(\bar s^i)\}\big);\forall i\in [N].\]
When $\mathcal G'$ encompasses all edges in $\mathcal G$, then $\mathbf P_{\mathcal G'}$ recovers $\mathbf P$, otherwise, $\mathbf P_{\mathcal G'}$ is a marginal probability of $\mathbf P$ whose marginalization is edges present in $\mathcal G$ but not in $\mathcal G'$. 
By definition, $\mathbf P_{\mathcal G'}$ is a transition kernel that admits $\mathcal G'$ as its dependence graph (it is, in fact, the best approximated $\mathbf P'$ on $\mathcal G'$).
Pinsker's inequality gives us an upper bound
\begin{align}
    \mathbb E\big\| P^i - P^i_{\mathcal G'}\big\|_1\leq \sqrt{2} \sum_{j, n=1}^N \mathbb E\Big[ \sqrt{I\big(S^i{'}; (S^{j}, A^j)\big)} 
    \Big| \text{Pa}^i_{n}(s^i)\bigg]
    \label{eq:mutual_info_bound}
\end{align}
where $\{\text{Pa}_n^i\}_{n=1}^N$ is a collection of non-decreasing agent-dependence sets for each $s^i$ such that $\text{Pa}_1^i(s^i) = \text{Pa}^i_{\mathbf P}(s^i) \cap\text{Pa}_{\mathcal G'}^i(s^i)$, $\text{Pa}_N^i(s^i) = \text{Pa}^i_{\mathbf P}(s^i) $, and $|\text{Pa}^i_{m+1}(s^i) \setminus \text{Pa}^i_{m}(s^i)|\leq 1$.
Furthermore, with some abuse of notations, we denote $\text{Pa}^i(\cdot)$ the states and actions of agents in $\text{Pa}^i(\cdot)$.
Full derivation and other details of this result are left to Appendix \ref{app:proof}.
Intuitively, if agent $i$’s states and actions have little impact on agent $j$'s future states, then the mutual information is low, and thus agent $j$ can be considered not to be connected to agent $i$ in the dependence graph. 
As a result, we can construct $\mathcal G'$ so that the upper bound in \eqref{eq:mutual_info_bound} is small, leading to small error gradients as in Lemma \ref{lem:graph-approx}.

\subsection{Approximating Dependence Graph via reverse world models}\label{sec:reverse_model}
We conclude our method by proposing a simple approach for learning dependence graphs that circumvents learning a forward world model.
To make the estimates in \eqref{eq:mutual_info_bound} practical, we introduce two simplifications to ease the mutual information estimation problem.

First, the mutual information in (\ref{eq:mutual_info_bound}) can be written as $I(S^{i}{'}; S^j, A^j) = I(S^{i}{'}; S^j) + I(S^{i}{'}; A^j | S^j)$.
As the MI between states $S^i{'}$ and $S^j$ can be difficult to estimate, we will ignore the former term and assume that the next states only depend on the actions of other agents. Second, we approximate the states and actions in $\text{Pa}_n^i(s^i)$ defined in \eqref{eq:mutual_info_bound} with $(s^i, a^i)$. This approximation is accurate when the dependence graph has a low degree of vertices, such as 1 or 2. 
Overall, with these simplifications, our objective derived from (\ref{eq:mutual_info_bound}) becomes
\begin{align}
    L_i(\boldsymbol{\pi}) := \mathbb E\sum_j^N I(S^i{'}; A^j| S^i, A^i, S^j)
    = \mathbb E\sum_j^N \left[ H(A^j| S^j) - H(A^j|S^i, S^j, S^i{'}) \right].
    \label{eq:objective}
\end{align}
To estimate the entropies in \eqref{eq:objective}, we adopt a variational lower bound based on a learned reverse dynamics model. Specifically, we train (i) an action prediction model $q_{\phi}$ to estimate $\hat H_{q_{\phi}}(A^j \mid S^j)$ and (ii) a multi-agent reverse world model $q_{\psi}$ that predicts agent $j$’s action conditioned on the state information of both agents $i$ and $j$.
To determine whether a directed edge from agent $j$ to agent $i$ exists at state $\mathbf{s}$, we compute the discrepancy between the two entropy terms. 
If conditioning on agent $i$ significantly reduces the uncertainty of $A^j$, then we consider there is an edge between them.
To construct $\mathcal{G}^{'}$, we apply a soft thresholding as a heuristic rule to define an edge from $j$ to $i$ whenever
\[\frac{\hat H_{q_{\psi}}(A^j|s^i, s^j, s^{i}{'})}{\hat H_{q_\phi}(A^j|s^j)} < c,\]
{where $c$ is a hyperparameter in the range of $[0, 1]$}.
The value of $c$ can be used to trade off the density of the approximated graph $\mathcal{G}^{'}$ for accuracy.
At the extreme value $c=0$, our method reverts to the standard local reward setup.
In our experiments, we manually select $c$ at 0.9 for all experiments.

To improve the effectiveness of our approach, we additionally learn an encoder and incorporate  the coordination graph into the GAE estimation, details and motivation designs are left to appendix \ref{sec:practical}.

%% file: experiment.tex
\section{Experiments}

\begin{wrapfigure}[21]{r}{0.5\textwidth}
\includegraphics[width=0.495\textwidth]{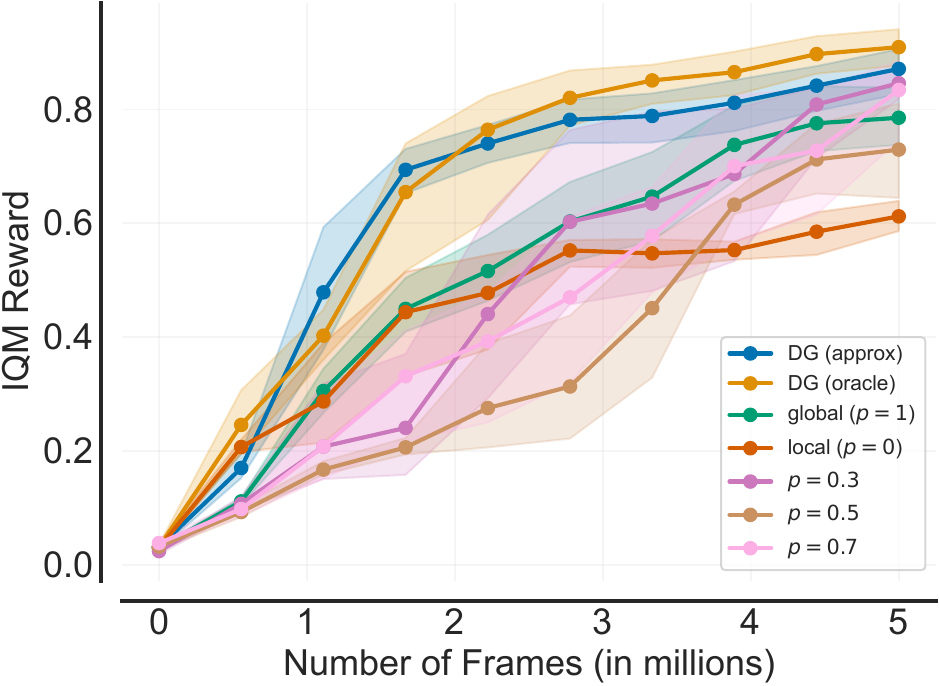}
\caption{We test our methods (DG) with various random dependence graph structures, where each agent has a probability of $p$ to form an edge with any other agent at any state. Global and local reward settings can be considered as special cases of random graphs with $p=1$ and $p=0$, respectively.}\label{fig:lbf_er_graph}
\end{wrapfigure}
We evaluate our methods on three MARL benchmarks, namely MPE, LBF, and SMAClite, which we modified to support local reward settings. Due to space constraints, we defer the environment details to Appendix \ref{appx:env}. Since our method is compatible with any on-policy gradient method, 
we adopt our backbone algorithms on both MAPPO \cite{yu2022surprising} and IPPO \cite{de2020independent}. Our source code is publicly available \footnote{\url{https://github.com/giangbang/dependence-graph-epymarl}}.

\textbf{Star-sread MPE.}
\begin{figure*}[]
\centering
\includegraphics[width=0.99\textwidth]{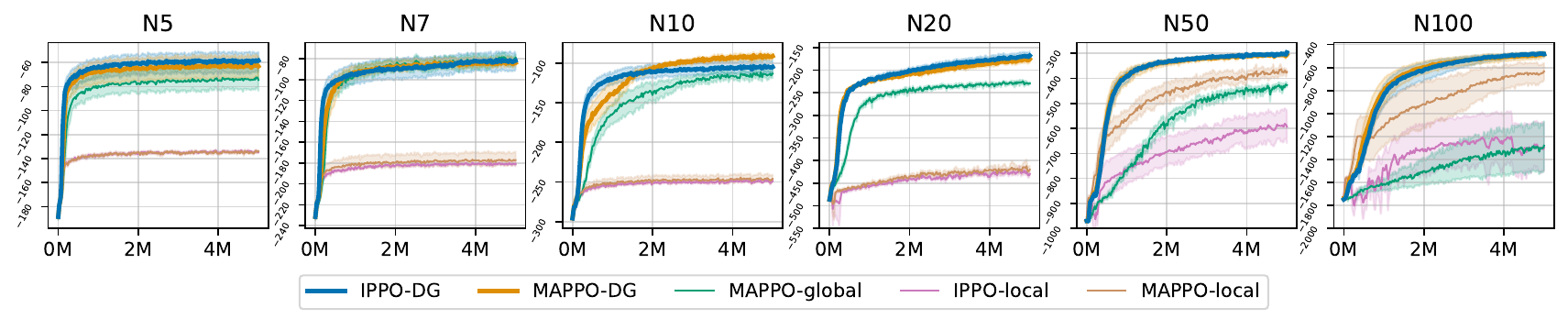}
\caption{Results on the Spread MPE with known dependence graphs (DGs) and varying numbers of agents $N$. Shaded regions denote 95\% stratified bootstrap confidence intervals. Global rewards perform well for small $N$ but scale poorly as $N$ increases, whereas local rewards scale better. 
Our method combines the advantages of both, maintaining strong performance across scales.
}
\label{fig:mpe}
\end{figure*}
In the first part of the experiments, we evaluate our method under a known Dependence Graph (DG) setting. We design a simple MPE environment with a star-structured reward graph that is sufficiently sparse while still requiring coordination among agents. Figure \ref{fig:mpe} demonstrates the effectiveness of our approach, with IPPO-DG and MAPPO-DG achieving competitive performance compared to both local- and global-reward baselines. Moreover, the performance gap between the DG-based methods and the global-reward baseline widens as the number of agents increases, which supports our theoretical analysis.

\textbf{LBF with learned graphs.}
\begin{figure*}[]
\centering
\includegraphics[width=\textwidth]{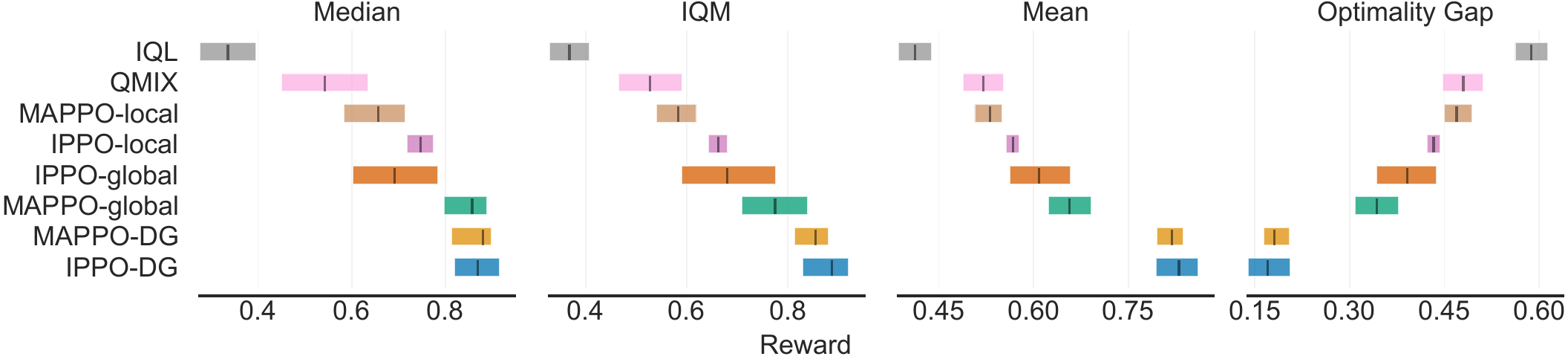}
\caption{Results on the LBF benchmark on 6 selected scenarios following the evaluation protocol in \citet{agarwal2021deep}. Methods with (learned) dependence graphs consistently outperform other baselines from both local and global reward settings.}
\label{fig:lbf_rliable}
\end{figure*}
Figure \ref{fig:lbf_rliable} presents the results on the LBF benchmark. Our method improves performance for both IPPO and MAPPO, achieving the best results among all baselines. Specifically, using the Dependence Graph, IPPO increases its interquartile mean (IQM) reward to 0.89, compared to 0.66 with local rewards and 0.68 with global rewards. MAPPO shows a similar trend: IQM improves to 0.85 with the Dependence Graph, compared to 0.58 under local rewards and 0.77 under global rewards. 
Dependence Graph yields slightly larger gains for IPPO than for MAPPO due to its stronger performance under the local reward setup.
Full experimental results are provided in Appendix \ref{app:detail_exp}.

We conduct additional studies on MAPPO with different graph structures in LBF. 
As baselines, we choose random graph structures in which each agent has a probability $p$ of connecting to the other agents. We tested with different values of $p$ corresponding from sparse to dense graph connections. Furthermore, we also construct a (heuristic) oracle dependence graph setup based on domain knowledge of the LBF transitions. 
Specifically, we consider two agents connected if their L1 distances are less than or equal to 2. If there are multiple such agents, we only keep the two closest ones. Figure \ref{fig:lbf_er_graph} demonstrates the effectiveness of our graph approximation, which performs almost comparably with the oracle graphs.

%% file: conclusion.tex
\section{Limitations and Conclusions}
\textbf{Limitations.}
In environments where the oracle and heuristic graphs are available, such as in MPE and LBF, our methods outperform both local and global approaches. Our method with learned graph works almost as the oracle graph in LBF, but does not improve on SMAClite. This reflects the main limitations of our approach:
\begin{itemize}
    \item Our entire work assumes that the graph depends on local states ($\text{Pa}(s)$ in section \ref{sec:settings}). In particular, the learned graph section was designed specifically that the graph can be inferred from the local states alone. What will happen when this assumption does not hold remains open. Unfortunately, this assumption can be difficult to hold in general (we discuss this in appendix \ref{app:assumption}). Modeling such non-local information can require some form of communication or central mediators. However, given that our graph formulation is already more general than previous works \citep{bagnell2005local, qu2019exploiting, qu2020scalable, jing2024distributed}, we consider it reasonable to lewave it to future work.
    \item Our learned graph procedure makes aggressive simplification on the state dependence. In particular, it may fail when transition dynamics depend more on agents' states than on their actions. Nevertheless, to the best of our knowledge, our work is also the first to attempt to learn the dependence graph from data. Furthermore, the reverse model is only able to learn information about the action, and we believe that capturing the state information requires some form of forward world model. Again, we leave this direction to future work.
\end{itemize}

\textbf{Conclusion.} In this work, we present a MARL framework that leverages dependence graphs 
that capture dependencies between agents and their local states 
to filter irrelevant reward signals and mitigate credit assignment. 
We also propose a practical method that remains effective under noisy approximated graphs. 
Experiments show that our approach performs competitively across both local and global reward settings, offering a robust balance between the two approaches.

%% file: appendix.tex
\clearpage

\appendix
\section*{Appendix}\label{appx:a}

\input{related}

\section{Further details omitted in the main text}
\subsection{Multi-agent Policy Gradient}\label{sec:multi_agent_pg}
The extension of the policy gradient theorem to multi-agent learning is introduced and proved in \citep{foerster2018counterfactual}, which we reproduce here for completeness. 
From the single-agent policy gradient theorem, we have
\begin{align*}
    \nabla_{{\pi_i}} J(\boldsymbol{\pi}) &= \mathbb E [\nabla_{\pi_i} \log \boldsymbol{\pi}(\boldsymbol{a}|\mathbf s)Q(\mathbf s, \boldsymbol{a})]\\
    &= \mathbb E \left[\nabla_{\pi_i} \log\prod \pi_i(a^i|s^i)  Q(\mathbf s, \boldsymbol{a})\right]\\
    &=\mathbb E \left[\sum_i^N\nabla \log \pi_i(a^i|s^i)  Q(\mathbf s, \boldsymbol{a})\right]
\end{align*}
where the second equality is based on the independent policy assumption. With weight sharing, the global gradient equals the sum of the gradients of individual agents. In case that we want to write the gradient of one agent to another, as in Eq. \eqref{eq:cross-grad}, we can set $r=r^i$; the global reward as the reward of agent $i$, and using the multi-agent policy gradient theorem as above.

\input{discussion}
\begin{wrapfigure}[15]{r}{0.35\textwidth}
\includegraphics[width=0.9\linewidth]{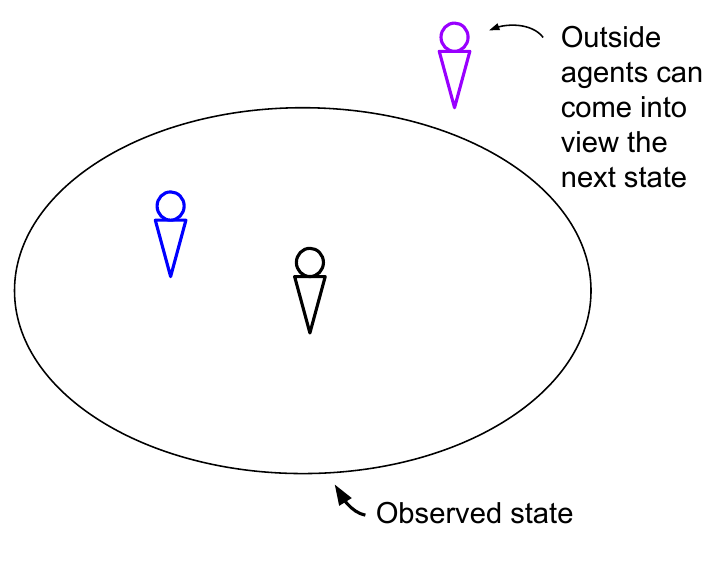} 
\caption{A simplified example, the purple agent can go into the observed states of the black agent and cannot be captured by a dependency graph based on the current substates of the black agent.}
\label{fig:dec_assumption}
\end{wrapfigure}
 
\subsection{Factorization of states in practice}\label{app:assumption}
Our theoretical ground for the graph approximation procedure is based on the assumption that an agent can infer their neighbor agents from observed states, as introduced in Section \ref{sec:settings}. This assumption naturally allows for the graph approximation idea based on local world model learning. 
In essence, the decomposition graph requires that current substates contain enough information to determine which agents it can interact with that can influence its next states.
Strictly speaking, this requirement can be difficult to meet in practice. We demonstrate this argument in an illustrative example in this section.
More specifically, we consider a simplified environment in Figure \ref{fig:dec_assumption}. 
In this example, the states of each agent are the field of vision where it can observe which other agents are adjacent in their vicinity (the oval circle). Next states of one agent not only depend on the nearby agents, but can also depend on the outside agents that can jump into its vision. A decomposition graph that is built on this observation state alone can not fully capture such "outsider" influences. However, we can consider that their overall effect can be negligible if the events of outside agents' interference happen rarely. 

\subsection{Practical Implementation}\label{sec:practical}
\begin{wrapfigure}[15]{r}{0.5\textwidth}
\centering
\includegraphics[width=0.49\textwidth]{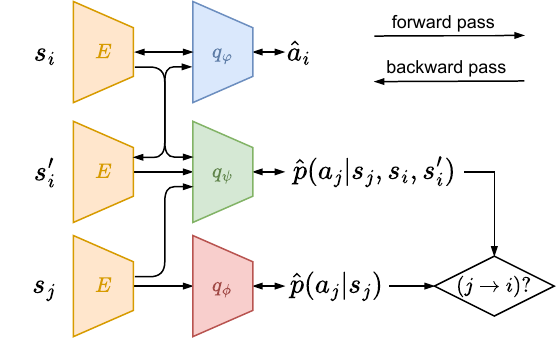}
\caption{Dependence graph approximation via reverse world models.}\label{fig:method}
\end{wrapfigure}
To make our method more effective in practice, two additional challenges must be addressed. First, the approximated graph on the original state space may contain many spurious edges. 
For example, agents may partially observe other agents' local states, even when it does not have a direct impact on their future states; such observations can allow for more accurate prediction of other agents' actions, resulting in many redundant edges.
Second, the approximated entropy difference in \eqref{eq:objective} is neither a lower nor an upper bound of the mutual information, as discussed in \citet{mcallester2020formal}, which can lead to noisy graph approximations. 

To mitigate the first issue, we adopt a learned latent state that filters out irrelevant information. Specifically, we train an encoder $E$ that maps the original states to a latent space that captures the information relevant only to the current agent by using a single-agent reverse world model $q_{\varphi}$. This single-agent reverse world model guides the encoder to retain only the information that is necessary for the current agent while ignoring the extraneous information about others. Our overall model approximation scheme with reverse world models is illustrated in Figure \ref{fig:method}.

To address the second issue, we devise an estimate strategy that combines multiple meeting timesteps $t'$. In particular, Proposition \ref{prop:global} shows that the true meeting timestep $t'$ provides the tightest estimate; however, any other estimate using a timestep earlier than $t'$ will also have the same expected gradient. As a result, we compute an exponentially weighted average of the dependence-graph gradients over all meeting timesteps earlier than the approximated timestep $t'$. This aggregation scheme is analogous to the GAE scheme, which results in a novel GAE-like estimate that incorporates the dependence graphs. The overall procedure is presented in Algorithm \ref{alg:gae}.

\subsection{Multi-Agent GAE-$\lambda$ with dependence graph}
\input{gae}

\section{Proofs for section \ref{sec:dgpg}}
\label{app:proof}

\input{policy_gradient_proof}

\begin{proof} [Proof of Lemma \ref{lem:graph-approx}]
For fixed agents $i$ and \(j\), and a join policy $\boldsymbol{\pi}$, define the dependency-truncated value functions as
\begin{align}
    \label{eq:Q_P_G_def}
Q^{i, j}_{\mathbf P,\mathcal G}(\mathbf s,\mathbf a)
&:= 
\mathbb E_{\tau\sim \mathbf P,\ \boldsymbol\pi}
\Biggl[
\sum_{t=t'}^{\infty} \gamma^{t} r_t^i
\ \Big|\ \mathbf s_0=\mathbf s,\ \mathbf a_0=\mathbf a
\Biggr]\\
&\;= \mathbb E_{\tau\sim\mathbf  P,\ \boldsymbol\pi}
\Biggl[
\sum_{t=0}^{\infty} \gamma^{t} r_t^i 1(j\in \text{Pa}^i(s^i_t, t, \tau))
\ \Big|\ \mathbf s_0=\mathbf s,\ \mathbf a_0=\mathbf a
\Biggr], \label{eq:eq7}
\end{align}
where we use \(t':=T_{ji}(s_0^j, \tau)\) as a shorthand for the first timestep on trajectory \(\tau\) at which there exists a directed path in graph \(\mathcal G\) from the originating vertex \((\mathbf s_0,j)\) to \((\mathbf s_{t'}, i)\); if no such time exists we take the Q value to be zero, and $\text{Pa}^i(s_t^i, k, \tau)$ denotes the set of $k$-hop parents in the path to $(\mathbf s_t, i)$, i.e. $\text{Pa}^i(s^i_t, k, \tau)=\text{Pa}^i(s^i_{t}, k-1, \tau)\cup\big(\bigcup_{m\in \text{Pa}^i(s^i_{t}, k-1, \tau)}\text{Pa}^m(s^m_{t-k})\big)$ and $\text{Pa}^i(s^i_{t}, 0, \tau) = \{i\}$. 
Then, 
\[
\begin{aligned}
&\nabla_{\pi_j}\mathcal J^i(\boldsymbol\pi;\mathcal G) - \nabla_{\pi_j}\mathcal J^i(\boldsymbol\pi;\mathcal G') \\
&\qquad=
\mathbb E_{\rho_{\boldsymbol\pi},\boldsymbol\pi}\big[
\nabla\log\pi_j(a^j\mid s^j)\,\big(
Q^{i,j}_{\mathbf P,\mathcal G}(\mathbf s,\mathbf a) - Q^{i, j}_{\mathbf P,\mathcal G'}(\mathbf s,\mathbf a)
\big)
\big]\\
&\qquad= 
\mathbb E_{\rho_{\boldsymbol\pi},\boldsymbol\pi}\big[
\nabla\log\pi_j(a^j\mid s^j)\,\big(
Q^{i,j}_{\mathbf P,\mathcal G}(\mathbf s,\mathbf a) - Q^{i,j}_{\mathbf P',\mathcal G'}(\mathbf s,\mathbf a) + Q^{i,j}_{\mathbf P',\mathcal G'}(\mathbf s,\mathbf a) - Q^{i, j}_{\mathbf P,\mathcal G'}(\mathbf s,\mathbf a)
\big)
\big]
\end{aligned}
\]
Due to Proposition \ref{prop:global}, we have $\nabla \log \pi_j(a^j|s^j) Q^{i,j}_{\mathbf P,\mathcal G}(\mathbf s,\mathbf a) = \nabla \log \pi_j(a^j|s^j) Q^{i}_{\mathbf P}(\mathbf s,\mathbf a)$; where we denote $Q_{\mathbf P}$ the conventional Q value estimated on the transition kernel $\mathbf P$. 
As a result, using \(\|\nabla\log\pi_j\|\le B_j\) and a triangular inequality gives a standard bound
\begin{align*}
&\big\|
\nabla_{\pi_j}\mathcal J(\boldsymbol\pi;\mathcal G) - \nabla_{\pi_j}\mathcal J(\boldsymbol\pi;\mathcal G')
\big\|\\
&\qquad \le
B_j \;\bigg[ \underbrace{\sup_{\mathbf s,\mathbf a}\big|Q^i_{\mathbf P}(\mathbf s,\mathbf a) - Q^i_{\mathbf P'}(\mathbf s,\mathbf a)\big|}_{T_1}  + 
\underbrace{\sup_{\mathbf s,\mathbf a}\big|Q^{i, j}_{\mathbf P,\mathcal G'}(\mathbf s,\mathbf a) - Q^{i, j}_{\mathbf P',\mathcal G'}(\mathbf s,\mathbf a)}_{T_2}\big| \bigg]
\end{align*}

\paragraph{Bound on \(T_1\).}
The bound in $T_1$ is a standard result in model-based learning. In particular, from the simulation lemma (see e.g. Lemma 2.2 in \citet{agarwal2019reinforcement}), one obtains
\[
\sup_{\mathbf s,\mathbf a}\bigl| Q^i_{\mathbf P}(\mathbf s,\mathbf a) - Q^i_{\mathbf P'}(\mathbf s,\mathbf a) \bigr|
\le
\frac{\gamma\,\varepsilon\,R_{\max}}{(1-\gamma)^2}.
\]

\paragraph{Bound on \(T_2\).}
To bound $T_2$, we observe that given a trajectory $\tau$ and a proper graph $\mathcal G'$, then the estimated graph-truncated values do not depend on the underlying kernel; the difference only presents in the expectation with the difference in probability of trajectories.
We start from the recursive form of $Q^{i, j}_{\mathbf P, \mathcal G}$ in \eqref{eq:eq7}, as
\begin{align*}
Q^{i, j}_{\mathbf P, \mathcal G}(\mathbf s, \mathbf a) &=  \mathbb E_{\tau\sim\mathbf  P,\ \boldsymbol\pi}
\Bigl[
\sum_{t=0}^{\infty} \gamma^{t} r_t^i 1(j\in \text{Pa}^i(s^i_t, t, \tau))
\ \Big|\ \mathbf s_0=\mathbf s,\ \mathbf a_0=\mathbf a
\Bigr]\\
&= \mathbb E_{\tau\sim \mathbf P,\ \boldsymbol\pi}
\Bigl[r^i_0 1(j\in\text{{Pa}}^i(s_0^i, 0, \tau))+
\sum_{t=1}^{\infty} \gamma^{t} r_t^i 1(j\in \text{Pa}^i(s^i_t, t, \tau))
\ \Big|\ \mathbf s_0=\mathbf s,\ \mathbf a_0=\mathbf a
\Bigr]\\
&=\begin{multlined}[t]
    \mathbb E_{\tau\sim \mathbf P,\ \boldsymbol\pi}
\Biggl[r^i_0 1(j\in\text{{Pa}}^i(s_0^i, 0, \tau))+
\gamma \sum_{t=1}^{\infty} \gamma^{t-1} r_t^i \bigg[1(j\in \text{Pa}^i(s^i_t, t-1, \tau)) \\
+ 1\bigg(j \in \bigcup_{k\in \text{Pa}^i(s^i_t, t-1, \tau)}\text{Pa}^k(s^k_{0})\bigg)\big(1 - 
1(j\in \text{Pa}^i(s^i_t, t-1, \tau))\big)\bigg]
\ \Big|\ \mathbf s_0=\mathbf s,\ \mathbf a_0=\mathbf a
\Biggr]
\end{multlined}\\
&= \mathbb E \bigg[ r^i_0 1(j \in \text{Pa}^i(s^i_0, 0, \tau)) + \gamma Q^{i, j}_{\mathbf P, \mathcal G}(\mathbf s_1, \mathbf a_1) + \gamma^{t'}r_{t'} \Big|\ \mathbf s_0=\mathbf s,\ \mathbf a_0=\mathbf a \bigg].
\end{align*}
The last equality uses the definition of $Q^{i, j}$, and an observation that $1\Big(j \in \cup_{k\in \text{Pa}^i(s^i_t, t-1, \tau)}\text{Pa}^i(s^k_{0})\Big)\Big(1 - 
1\big(j\in \text{Pa}^i(s^i_t, t-1, \tau)\big)\Big) =1$ only if $j$ is not in the $(t-1)$-hop parents to $s^i_t$, but is in the $t$-hop parents, i.e. the definition of $t'$.
Then, 
\begin{align*}
    &Q^{i, j}_{\mathbf P,\mathcal G'}(\mathbf s,\mathbf a) - Q^{i, j}_{\mathbf P',\mathcal G'}(\mathbf s,\mathbf a) \\
    &\quad= \mathbb E[r^i|\mathbf s, \mathbf a] 1\big( j\in \text{Pa}^i( s^i, 0) \big) + \gamma \sum_{\mathbf s'}\mathbf P(\mathbf s'|\mathbf s, \mathbf a) \big)
    \sum_{\mathbf a'}\boldsymbol{\pi}(\mathbf a'| \mathbf s')
    Q^{i, j}_{P, \mathcal G'}(\mathbf s', \mathbf a') + \mathbb E_{\mathbf P, \boldsymbol{\pi}}[\gamma^{t'}r_{t'}|\mathbf s, \mathbf a] \\
    &\quad \qquad -\mathbb E[r^i|\mathbf s, \mathbf a] 1\big( j\in \text{Pa}^i (s^i, 0) \big) - \gamma \sum_{\mathbf s'}\mathbf P'(\mathbf s'|\mathbf s, \mathbf a)
    \big)
    \sum_{\mathbf a'}\boldsymbol{\pi}(\mathbf a'| \mathbf s')
    Q^{i, j}_{P', \mathcal G'}(\mathbf s', \mathbf a') - \mathbb E_{\mathbf P', \boldsymbol{\pi}}[\gamma^{t'}r_{t'}|\mathbf s, \mathbf a ]\\
    & \quad =
    \begin{multlined}[t]
        \gamma \sum_{\mathbf s'}\bigg[ 
    \mathbf P(\mathbf s'|\mathbf s, \mathbf a)\sum_{\mathbf a'}\boldsymbol{\pi}(\mathbf a'| \mathbf s')
    Q^{i, j}_{P, \mathcal G'}(\mathbf s', \mathbf a')
     - \mathbf P'(\mathbf s'|\mathbf s, \mathbf a)\sum_{\mathbf a'}\boldsymbol{\pi}(\mathbf a'| \mathbf s')
    Q^{i, j}_{P', \mathcal G'}(\mathbf s', \mathbf a')\bigg] \\+ \left[ \mathbb E_{\mathbf P, \boldsymbol{\pi}}[\gamma^{t'}r_{t'}|\mathbf s, \mathbf a] - \mathbb E_{\mathbf P', \boldsymbol{\pi}}[\gamma^{t'}r_{t'}|\mathbf s, \mathbf a]\right].
    \end{multlined}
\end{align*}
From the definition of $Q^{i, j}_{\mathbf P, \mathcal G}$; it is easy to see that $0 \leq Q^{i, j}_{\mathbf P, \mathcal G} \leq \frac{R_{\max}}{1-\gamma}$, then the first term of the above equation can be bounded by
\begin{align*}
    \gamma \|\mathbf P(\cdot| \mathbf s, \mathbf a) - \mathbf P'(\cdot|\mathbf s, \mathbf a)\|_1\frac{R_{\max}}{1-\gamma}\leq
    \frac{ \gamma \epsilon R_{\max}}{1-\gamma}.
\end{align*}
For the second term, we observe that the timestep $t'$ of a given trajectory $\tau$ is the same regardless of whether the underlying transition probability is $\mathbf P$ or $\mathbf P'$, 
because we use the same graph $\mathcal G'$ and the paths on this graph are unchanged. As a result, we can think of having both $\mathbf P$ and $\mathbf P'$ to have the same reward function, but this reward function is non-Markovian on the original state space because rewards are only bestowed on the first timestep when agents $i$ and $j$ are connected. 
For that reason, we cannot directly apply the simulation lemma argument as in $T_1$.
Fortunately, this can easily be fixed by augmenting the state space so that the rewards are Markovian on this new state space.
More specifically, for an agent $j$, we augment each state $\mathbf s$ with a vector of flags $\mathbf f\in \{0, 1\}^N$, indicating whether the agent $j$ and other agents have already been connected. Let $\tilde{\boldsymbol{s}}=(\mathbf s, \mathbf f)$. The augmented transition kernel can be defined as
\[\tilde{\mathbf s}_{t+1}=(\mathbf s_{t+1}, \mathbf f_{t+1}) , \quad \mathbf s_{t+1}\sim \mathbf P(\cdot|\mathbf s_t, \mathbf a_t), \quad \mathbf f^i_{t+1}=1\bigg(i\in  \bigcup_{k; \mathbf f^k_t=1}\text{Pa}^k(s^k)\bigg), \quad \mathbf f_0^k=\begin{cases}
    1 & k=j\\
    0 & k\neq j
\end{cases}.\]
Under this augmented state space, the reward $r_{t'}$ is Markovian,
\[r_{t'}(\tilde{\mathbf s}_t, \mathbf a_t, \tilde{\mathbf s}_{t+1}) = r(\mathbf s_t, \mathbf a_t) 1(\mathbf f^i_t=0 \;\land\;\mathbf f^{i}_{t+1}=1 ), \]
Here, the indicator term denotes that the agent $i$ is not connected to $j$ in the current step but is in the next step. As a result, we can apply the same argument as in the bound $T_1$ by using the simulation lemma to conclude that
\[\sup_{\mathbf s, \mathbf a}\Big|\mathbb E_{\mathbf P, \boldsymbol{\pi}}[\gamma^{t'}r_{t'}|\mathbf s, \mathbf a] - \mathbb E_{\mathbf P', \boldsymbol{\pi}}[\gamma^{t'}r_{t'}|\mathbf s, \mathbf a]\Big| \leq \frac{\gamma\epsilon R_{\max}}{(1-\gamma)^2}.\]
Combining all the bounds, we get
\[
\big\|
\nabla_{\pi_j}\mathcal J(\boldsymbol\pi;\mathcal G) - \nabla_{\pi_j}\mathcal J(\boldsymbol\pi;\mathcal G')
\big\| \leq \gamma \epsilon B_j R_{\max}\frac{3-\gamma}{(1-\gamma)^2} < \frac{3\gamma \epsilon B_j  R_{\max}}{(1-\gamma)^2}.
\]
\end{proof}

\input{mutual_infomation_bound}

\input{reward_dependence_graph}

\input{sample_analysis_proof}

\input{env_spec}

\subsection{Experimental Settings and Network Architecture}\label{app:hyperparams}
We adopt our backbone algorithms on both MAPPO \cite{yu2022surprising} and IPPO \cite{de2020independent} with Dependence Graph (DG). We compare our approach against the same algorithms under two other reward settings, namely global and local rewards.
Furthermore, we also compare our method to IQL (local reward), QMIX (global reward) \cite{rashid2020monotonic}, and QPLEX (global reward) \cite{wang2020qplex} as other strong value-based baselines from both global and local reward settings. For a fair comparison, all methods utilize parameter sharing between agents.

All the algorithms are based on the PyMARL framework \cite{papoudakis2020benchmarking}. As a result, all the hyperparameters used in this paper closely follow those reported in the framework. For a fair comparison, all PPO-based algorithms with different reward setups are GAE-enabled, with the GAE lambda $\lambda$ set to 0.95 as recommended in \cite{schulman2015high}. 
Since SMAClite is not tested in the PyMARL, we follow suit SMAClite original paper \cite{michalski2023smaclite}  to use the same hyperparameters of PyMARL for SMAC. 
For QMIX and QPLEX, we adopt the finetuned hyperparameters for SMAC, as reported in \cite{hu2021rethinking}, and used it in SMAClite.
The full set of hyperparameters used in PPO-based methods is provided as follows.
\begin{table}[h]
    \centering
    \begin{tabular}{lcccc}
    \toprule
    & SMAClite & LBF & MPE & \\
    \hline
    hidden dimension & 64 & 128& 64\\
    learning rate & 0.0005 & 0.0005 & 0.0005 \\
    reward standardisation & False &False&True\\
    network type & GRU &FC& FC\\
    entropy coefficient & 0.001 &0.001&0.01\\
    clip eps & 0.2 & 0.2 & 0.2 \\
    \bottomrule
    \end{tabular}
    \caption{All PPO-based methods, including our MAPPO-DG and IPPO-DG, use the same hyperparameters for each benchmark.}
    \label{tab:mappo}
\end{table}


All neural networks used in our experiments for the baseline methods, including the policy and value networks, follow the original PyMARL implementations without modification. In our method, the state encoder $E$ is implemented as a three-layer neural network with hidden dimensions of 64 and an output dimension of 64. The single-agent reverse world model $q_{\varphi}$, the action predictor model $q_{\phi}$, and the multi-agent reverse world
model $q_\phi$ share the same architecture: a four-layer neural network with 256 neurons in the first hidden layer and 128 neurons in each of the remaining hidden layers. All networks use ReLU activations and Layer Normalization \cite{ba2016layer}.


\subsection{Detailed Experiment Results}\label{app:detail_exp}

\begin{figure}[t]
\centering
\includegraphics[width=\textwidth]{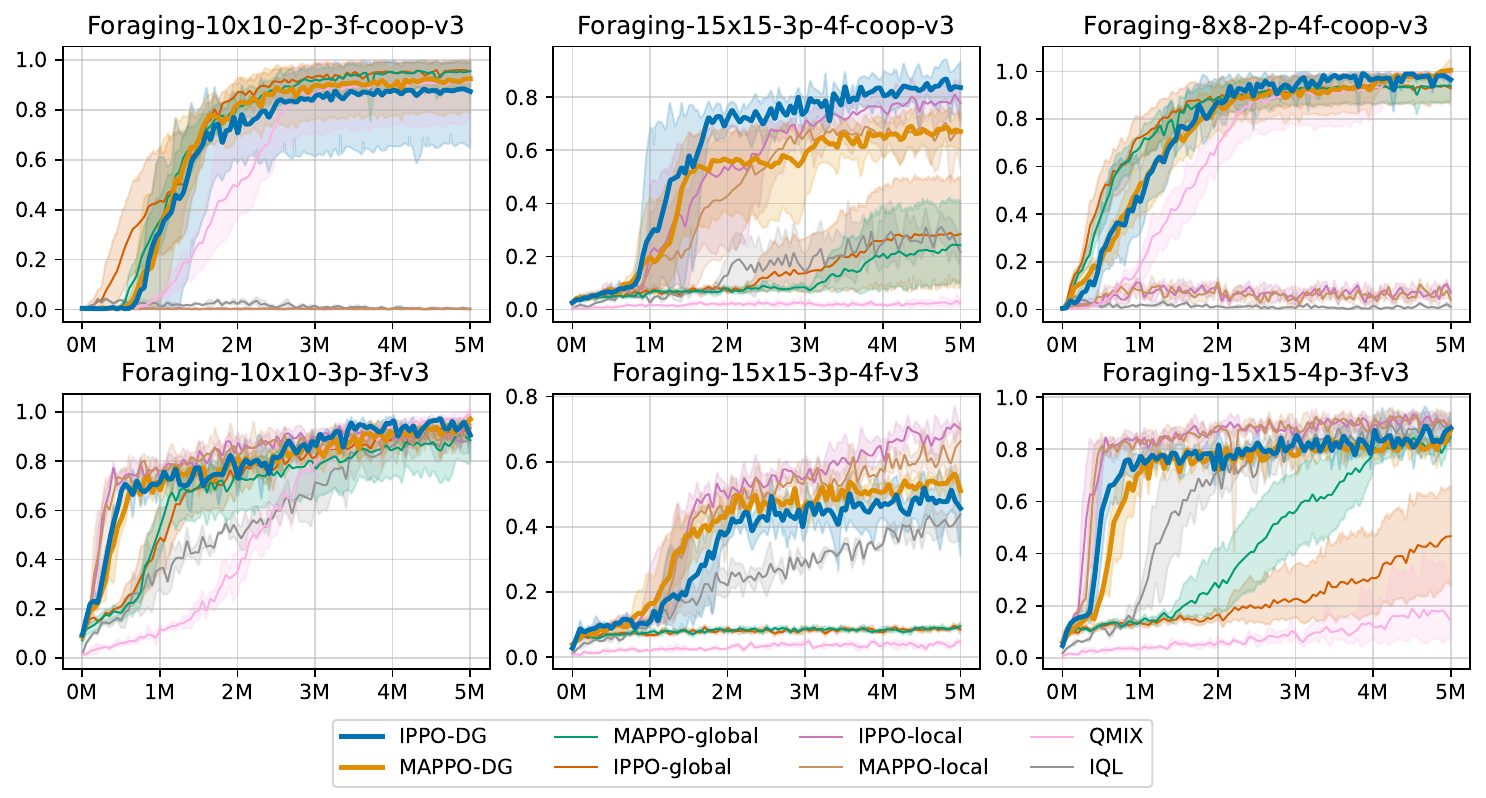}
\caption{Full results on the LBF benchmark.}
\label{fig:lbf_res}
\end{figure}

Figure \ref{fig:lbf_res} presents the complete results corresponding to Figure \ref{fig:lbf_rliable} in the main paper.
We can see that in the vanilla LBF scenarios (bottom row), agents trained with local reward signals consistently outperform those using global rewards. 
Especially in the 15x15-3p-4f environment, where all global methods, including MAPPO-global, QMIX, and IPPO-global struggle significantly due to the sparsity of the reward signals. 
In the 15x15-4p-3f setting, the addition of extra agents alleviates the sparsity issue to some extent, possibly due to the parameter sharing mechanism. 
MAPPO-global is able to converge to an optimal policy, albeit at a slower rate and with higher variance compared to the local reward setup. QMIX, on the other hand, fails to learn effectively even in this slightly less sparse scenario.
IQL using local reward in this setting outperforms QMIX with global rewards on all vanilla LBF setups except for in Foraging-10x10-3p-3f, where they both converge to optimal policies, but IQL learns a bit slower.
In contrast, the winner-take-all with the cooperative setup (top row) poses challenges to the local reward approach. 
In this setup, due to the competitive nature of the reward function, the local reward approach struggles; it fails to learn in 2 out of 3 scenarios. 
Interestingly, some of the completely different branches of algorithms behave quite similarly under the same reward setting. In particular, in 10x10-2p-3f-coop and 8x8-2p-4f-coop, all local reward algorithms fail, including on-policy methods with both MAPPO and IPPO, value-based method IQL, and value decomposition QMIX.
In 15x15-3p-4f-coop, we observe a similar trend as in the winner-take-all setup, where an additional agent improves the performance of local reward with both value-based and on-policy methods, except for QMIX. 
On the other hand, our method remains competitive across all the scenarios, demonstrating its robustness across different setups.

\begin{figure}[t]
\centering
\includegraphics[width=\textwidth]{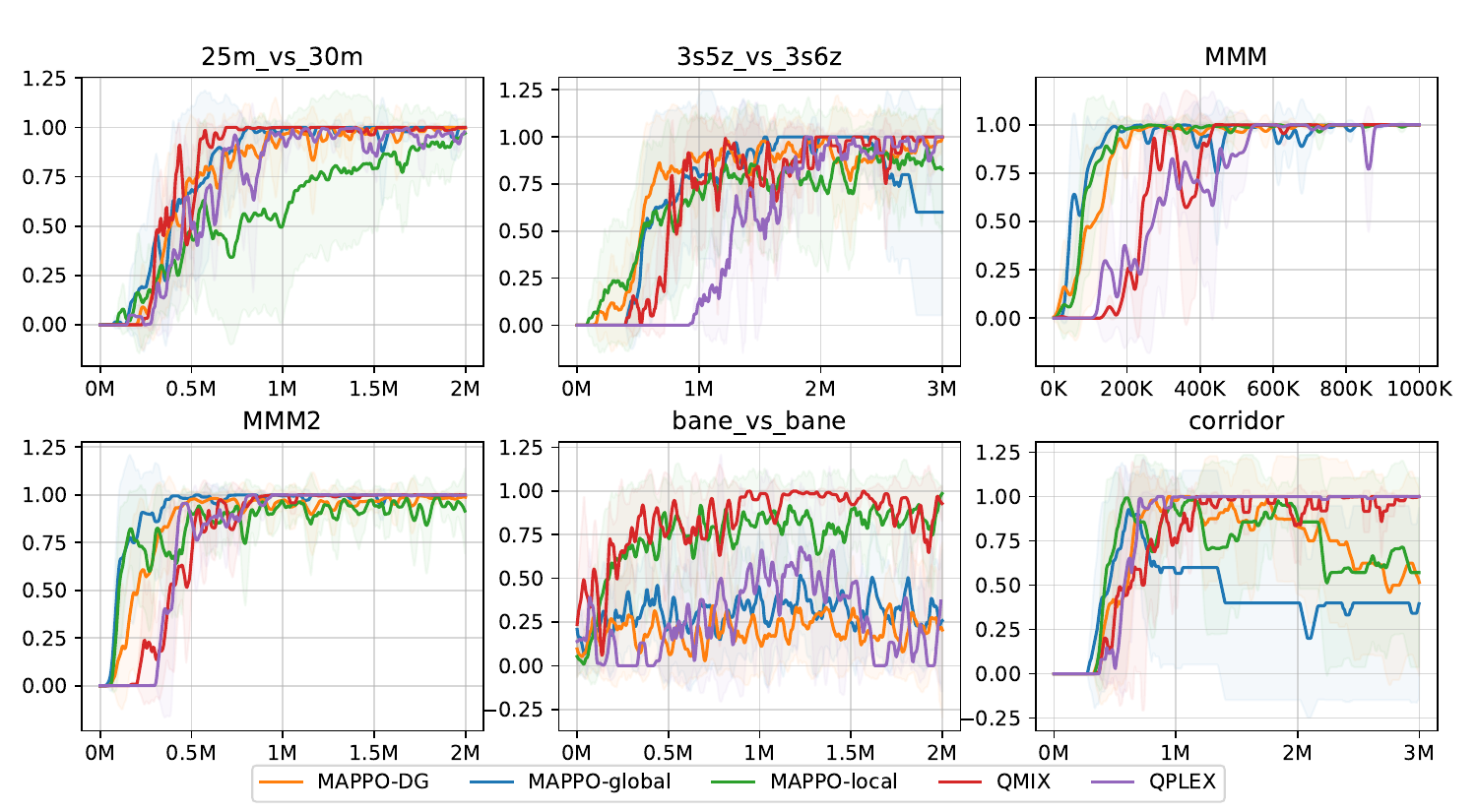}
\caption{Full results on the SMAClite benchmark.}
\label{fig:smaclite_res}
\end{figure}

\begin{figure}
\centering
\includegraphics[width=0.48\textwidth]{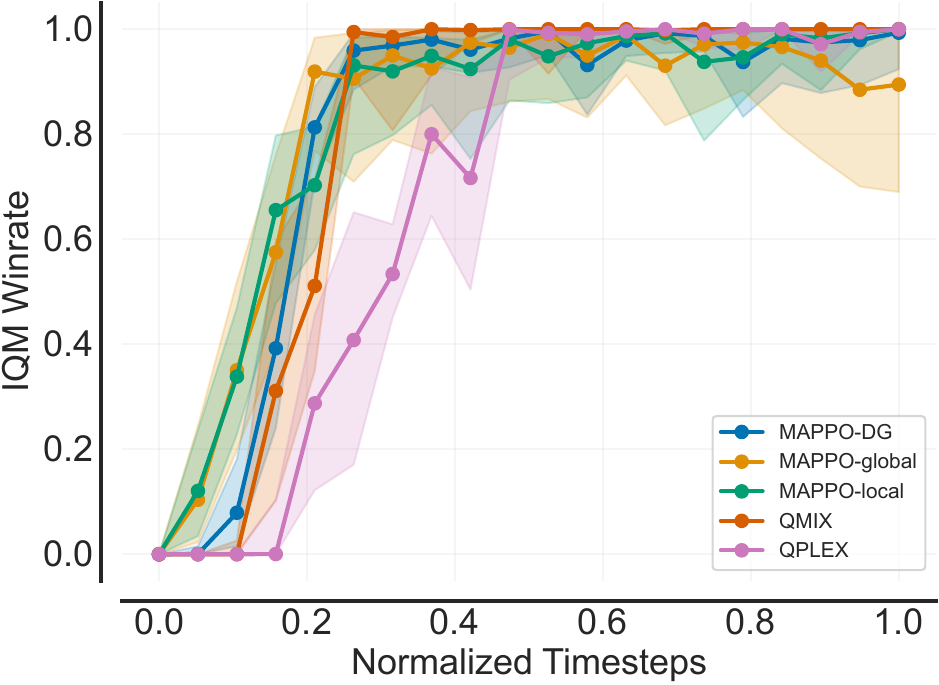}
\caption{Result on the SMAClite benchmark on 6 scenarios.}\label{fig:smaclite_rliable}
\end{figure}
Figures (\ref{fig:smaclite_res}) and (\ref{fig:smaclite_rliable}) present the full results on the SMAClite benchmark.
QMIX achieves consistently strong performance across all scenarios. QPLEX performs quite similarly to QMIX, but slower in 3s5z\_vs\_3s6z and does not learn in bane\_vs\_bane. 
In most scenarios, our method and global MAPPO achieve comparable performance; for example, in 25m\_vs\_30m, both outperform local MAPPO.On corridor, however, our method and local MAPPO outperform global MAPPO. We attribute the unstable learning behavior observed in all MAPPO-based methods to suboptimal hyperparameters adopted from PyMARL.

\subsection{Algorithm}
\input{mappo_pseudocode}

%% file: related.tex
\section{Related Work}\label{sec:relate}

\textbf{Reward for Cooperation. }
The behaviors of RL agents are defined by the reward structure, a principle known as the {reward hypothesis} \citep{sutton2018reinforcement}.
Effective reward functions are notoriously difficult to design, and this remains true in cooperative MARL \citep{pan2022effects}.
Traditionally, agents are trained using separate, agent-specific reward vectors under the {independent learning} framework, where each agent is treated as an individual RL problem \citep{tan1993multi, de2020independent}.
This approach offers benefits such as simplicity and scalability \citep{oroojlooy2023review}, but also suffers from issues like non-stationarity and miscoordination \citep{hernandez2017survey}.
The degree of cooperation is also sensitive to reward shapings \citep{leibo2017multi, tampuu2017multiagent}.

In contrast, using a shared reward function directly enforces cooperation as in fully cooperative MARL settings \citep{yuan2023survey}.
This setup is often trained under the Centralized Training with Decentralized Execution (CTDE) paradigm \citep{foerster2018counterfactual}, which further enhances coordination through centralized training.
However, shared rewards introduce the credit assignment problem, making it difficult to scale fully cooperative learning to environments with many agents \citep{bagnell2005local, wang2020breaking}.
As a result, local reward learning is more suitable for large-scale MARL problems \citep{ zheng2018magent}.

Recently, a number of approaches attempt to explore a middle ground between the two aforementioned reward settings. For instance, \citet{le2025toward} apply tools from multi-objective optimization to identify Pareto optimal policies. However, their method treats the problem as a generic optimization task and does not explicitly model individual agent contributions within the sequential decision-making process. 
\citet{wang2022individual} optimize auxiliary local rewards alongside the sparse global team reward. However, ultimately this approach still solves the two different reward formulations and relies on additional heuristic design of local rewards to enable effective learning.

\textbf{Credit assignment. } Credit assignment problem is the problem of attributing the contribution of individual agents from the team's rewards \citep{chang2003all}. 
Two main approaches to solving credit assignment problems in the literature include (i) decomposing the global value function, and (ii) marginalizing the contribution of one agent over the others \citep{gronauer2022multi}. 
The first approach, sometimes referred to as implicit credit assignment  \citep{zhou2020learning}, has gained popularity from the works of VDN \citep{sunehag2017value} and QMIX \citep{rashid2020monotonic}. 
This approach is characterized by using individual utility functions and a mixing procedure to combine them into a total joint value function. To facilitate efficient decentralized execution, most methods follow the Individual-Global-Max (IGM) principle \citep{hong2022rethinking}. 
IGM ensures local greedy agents align with globally optimal actions, but in practice, its limited decomposability can hinder convergence \citep{wang2020towards, rashid2020weighted}. 
Coordination Graph \cite{guestrin2001multiagent} provides another approach that factorizes the joint value function using a graph structure, enabling a higher-order value factorizations \citep{kok2006collaborative}. However, this method introduces significant computational overhead, which in practice often restricts implementations to pairwise interactions \cite{bohmer2020deep}. 
The second approach, with representative works such as Counterfactual learning \citep{foerster2018counterfactual}, mainly uses baselines \citep{weaver2013optimal, wu2018variance} to reduce the variance of Policy Gradient estimation \citep{kuba2021settling}.
Shapley Q-value \citep{wang2020shapley} is a similar approach that assigns credit based on agents' average marginal contribution over all coalition formations.
A common characteristic of these methods is that they start from the global reward and distribute credit in a top-down manner.
In contrast, our work avoids the credit assignment problem by directly using local rewards in a bottom-up approach. 

\textbf{Networked Multi-Agent MDPs. }
Theoretically, our formulation is related to the works that model agent interactions as a networked graph in the Multi-Agent MDP framework \citep{bagnell2005local, qu2019exploiting, qu2020scalable}. 
This formulation enables more fine-grained modeling of local interactions, which is useful for analyzing agents' contributions. 
Most works in this direction rely on the exponential decay property of value representations \citep{qu2020scalable} to approximate the value function using a small subset of agents, possibly with communications \citep{ma2024efficient}, while our work explores the time dimension of value estimation along the dependency graph paths.
Recent works \citep{jing2024distributed, syed2026structured} further show that the policy gradient depends only on the reachable subset of agents, which corresponds to the infinite meeting-time case in our framework.
Another related direction is explored in \citep{lin2021multi}, which generalizes the setting to stochastic graphs. While our work assumes deterministic graphs, the dependence paths on the graphs evolve with the states visited over time and can therefore be regarded as effectively stochastic due to the stochasticity of future trajectories.


%% file: discussion.tex
\subsection{Discussion and limitations}

\textbf{Compare with the fully cooperative setting.}
While our approach requires the additional local-reward information, which may not be available for all environments. 
It is common that the reward from the scalar reward team can be attributed to a specific agent. For example, in the LBF benchmark, the reward corresponds to agents participating in the food foraging. Even when the reward cannot be attributed explicitly, we can always distribute it evenly among the team, which essentially reduces our setting to a fully cooperative one. Additionally, scalarizing the rewards is a one-way operator, where any vector-valued reward can be transformed into a scalar team reward, but the converse is not possible. Furthermore, even when the individual reward signals are available, existing methods do not provide principled approaches to leverage this reward vector to mitigate the credit assignment problems. In contrast, our work offers a theory-grounded framework for incorporating individual reward information directly into the optimization of the team objective.

\textbf{Limitations.}
Our approach to graph approximation depends on the assumption that the neighboring agents can be inferred from the local observation, which technically is difficult to satisfy in practice (see Appendix \ref{app:assumption} for a discussion). Furthermore, the effect of ignoring the state mutual information in Section \ref{sec:reverse_model} is unknown, and we expect it to fail in environments where the transition probabilities depend heavily on agents' states instead of actions.

%% file: gae.tex
\label{app:gae}
The key idea behind GAE is to average over all the $n$-step return estimations, which balances the high variance of Monte Carlo sampling and the high bias of the short-horizon $n$-step returns. 
We employ the same ideas as GAE for the meeting timestep $t'$, where we do not exclusively rely on a single estimate of this quantity. 
Similar to GAE, we start from the definition of TD error, defined as $\delta_t=r_t+\gamma V^i(s_{t+1})-V^i(s_t)$, where $V^i(\cdot)$ are value estimates of the agent $i$'s policy. The multi-step advantage estimates from $\delta$, combined with the dependence graphs at each timestep, are then defined as follows, for a given trajectory $\tau$,
\begin{align*}
    \hat{A}_t^{(1)}&=\delta_t\\
    \hat{A}_t^{(2)}&={1}(j\in \text {Pa}^i({s_t^i, 1, \tau}))\delta_t + \gamma \delta_{t+1}\\
    &\vdots\\
    \hat{A}_t^{(k+1)}&=
    \sum_{l=0}^{k-1}\gamma^l{1}(j\in \text{Pa}^i (s^i_{t+l}, l, \tau))\delta_{t+l} + \gamma^{k}\delta_{t+k},
\end{align*}
where $\text{Pa}^i (s_{t}, t, \tau)$ is the $t$-hop parents of $s_t^i$, defined in \eqref{eq:k_hop_def}.
Each of the $\hat{A}_t^{(k)}$ involves the sum of $k$ TD error terms.
Furthermore, all $\hat{A}_t^{(k)}$ are equivalent to advantage estimation with meeting times $k$ upto $t'>k$, with $t'$ the predicted first timestep that $i$ and $j$ are connected by a path.
This can be effective for mitigating the noisy predictions of the dependence graphs inferred by the world models.
Similar to GAE, we take the exponentially weighted average of all these terms
\begin{flalign*}
    \hat{A}_t^{\text{GAE}}&=(1-\lambda) \left( \hat{A}_t^{(1)} +\lambda \hat{A}_t^{(2)} + \lambda^2 \hat{A}_t^{(3)} + \dots \right)\\
    &=\begin{multlined}[t]
        (1-\lambda) \Big( \delta_t +\lambda \big( 1 (j\in \text {Pa}^{i} ({\mathbf s^i_{t}}, t, \tau))\delta_t + \gamma \delta_{t+1}\big) \\
        +\lambda^2 \big( 1 (j\in \text {Pa}^{i} ({\mathbf s^i_{t}}, t, \tau))\delta_{t} + 
        \gamma 1 (j\in \text {Pa}^{i} ({\mathbf s^i_{t+1}}, t+1, \tau))\delta_{t+1} + \gamma^2 \delta_{t+2}\big)
        +\dots \Big)
    \end{multlined}\\
    &= \begin{multlined}[t]
        (1-\lambda)\Big(\delta_t\big(1 + 1\big(j \in \text {Pa}^{i}(\mathbf s_{t}^i, t, \tau)\big)(\lambda + \lambda^2+\dots)\big) \\
        +\gamma\delta_{t+1}\big(\lambda +  1(j \in \text {Pa}^{i}({\mathbf s_{t+1}^i, t+1, \tau}))(\lambda^2 + \lambda^3+\dots)\big) +\dots\Big)
    \end{multlined}\\
    &= \begin{multlined}[t]
        \lambda(1-\lambda)\Big(\delta_t  1(j \in \text {Pa}^{i}( s_{t}^i, t, \tau))(1 + \lambda+\lambda^2+\dots) \\
        +\gamma\delta_{t+1} 1\big(j \in \text {Pa}^{i}(s_{t+1}^i, t+1, \tau)\big)\big(\lambda + \lambda^2+\dots\big) +\dots\Big) \\ + (1-\lambda)\big(\delta_t + \gamma\lambda\delta_{t+1}+\dots\big)
    \end{multlined}\\ 
    &=\sum_{l=0}^\infty \gamma^l\lambda^{l+1}\delta_{t+l} 1\big(j\in\text {Pa}^{i} (s_{t+l}, t+l, \tau)\big) + (1-\lambda)\sum_{l=0}^\infty (\gamma\lambda)^l\delta_{t+l}\\
    &= \sum_{l=0}^\infty (\gamma \lambda)^l\delta_{t+l}\Big( \lambda1\big(j\in\text {Pa}^{i} (s_{t+l}, t+l, \tau)\big) + 1 -\lambda \Big).
\end{flalign*}
Intuitively, when two agents are not connected by a path, we do not discard the contribution of the other agent’s gradient entirely; instead, we downweight it by a factor of $1-\lambda$. 
Finally, to find the $t$-hop parent set, we perform successive matrix multiplications of the adjacency matrices to the current timestep $t$ as detailed in Algorithm \ref{alg:gae}. 
\input{gae_pscode}

%% file: gae_pscode.tex
\begin{algorithm}[]
\caption{Policy Gradient with Multi-Agent GAE-$\lambda$ and Dependence Graph}
\label{alg:gae}
\begin{algorithmic}[1]
\INPUT Trajectory $\tau=\{(\mathbf s_t, r^i_t)\}_{t=0}^{T-1}$, value estimates $\{V^i(\mathbf s_t)\}_{t=0}^{T}$,
discount factor $\gamma$, GAE parameter $\lambda$, adjacency matrix $\{\text{A}^t\}_{t=0}^{T-1}$, 
action scores $\{\nabla \log(\pi_j(a^j_t|s^j_t))\}_{t=0}^{T-1}$
\OUTPUT Gradient $\hat \nabla_{\pi_j}\mathcal J^i(\boldsymbol{\pi}; \mathcal G)$, Advantage $\{A^i_t\}_{t=0}^{T-1}$

\STATE Initialize $A_t \leftarrow 0$ for all $t = 0, \dots, T-1$

\FOR{$t_0 = T-1$ \textbf{down to} $0$}
\STATE $\text{gae} \leftarrow 0$
\STATE $\text{H} \leftarrow \text{diag}(\mathbf 1)$
\FOR{$t = t_0$ \textbf{to} $T-1$}
    \STATE $\delta^i_t \leftarrow r_t^i + \gamma V^i(\mathbf s_{t+1}) - V^i(\mathbf s_t)$
    \STATE $\text{H} \leftarrow \text{H} \cdot \text{A}^t$
    \STATE $\text{gae} \leftarrow \delta^i_t (\gamma \lambda)^{t-t_0} \big(\lambda 1(\text{H}_{ji}>0)  + 1 - \lambda \big) +  \text{gae}$
\ENDFOR
\STATE $A_{t_0}^i \leftarrow \text{gae}$
\ENDFOR

\RETURN $\frac{1}{T}\sum_{t=0}^{T-1}\gamma^t A_t^i\nabla \log \pi_j(a_t^j|s_t^j) $, $\;\;\{A^i_t\}_{t=0}^{T-1}$
\end{algorithmic}
\end{algorithm}

%% file: policy_gradient_proof.tex
In this section, we use the following definition of path on the dependence graphs.
\begin{definition}[Path in the Dependence Graph]\label{def:path}
Let $\mathcal G = \langle \mathcal V, \mathcal E\rangle$ be the induced Dependence Graph of $\mathcal M$.
A \emph{path} from vertex $(\mathbf s, i)$ to vertex $(\mathbf s', j)$
is a sequence of vertices
\[
(\mathbf s_0, i_0) \;\to\; (\mathbf s_{1}, i_1) \;\to\; \cdots \;\to\; (\mathbf s_{L}, i_L),
\]
such that
\begin{enumerate}
    \item $i_0 = i$, $i_L = j$, $\mathbf s_0=\mathbf s$, $\mathbf s_L = \mathbf s'$,
    \item for each $\ell=0,\dots,L-1$, the directed edge 
    $(\mathbf s_{\ell}, i_\ell) \to (\mathbf s_{\ell+1}, i_{\ell+1})$
    belongs to $\mathcal E$.
\end{enumerate}
\end{definition}
We call $L$ the \emph{length} of the path. 
If there is a path from $(\mathbf s_t, i)$ to $(\mathbf s_{t'}, j)$, then there is a path from $(\mathbf s_t, i)$ to any of $(\mathbf s_{{t''}}, j)$ for all ${t''} \geq t'$ since an agent influences its next states. A path is defined on a realizable sequence of states that we call a (state) trajectory.
\begin{proof}[Proof of Proposition \ref{prop:global}]
    From the policy gradient theorem, we unroll each timestep
    \begin{align}
        (1-\gamma)\nabla_{\pi_j}J^i(\boldsymbol{\pi}) &= \mathbb E_{\mathbf s_0 \sim \rho_{\boldsymbol{\pi}},\mathbf a_0\sim  \boldsymbol{\pi}}[\nabla \log \pi_j(a_0^j|s_0^j)Q^i(\mathbf s_0, \mathbf a_0)]\notag \\
        &=  \mathbb E_{\rho, \boldsymbol{\pi}}\bigg[\nabla \log \pi_j(a_0^j|s_0^j)\Big[\mathbb E[r^i_0(s_0^i, a_0^i)] + \gamma \mathbb E_{\mathbf s_1, \boldsymbol{\pi}}\big[Q^i(\mathbf s_1, \mathbf a_1)|\mathbf s_0, \mathbf a_0\big] \Big]\bigg] \notag\\
        &= \mathbb E_{\rho, \boldsymbol{\pi}}\Big[\nabla \log \pi_j(a_0^j|s_0^j)\gamma \mathbb E_{\mathbf s_1, \boldsymbol{\pi}}\big[Q^i(\mathbf s_1, \mathbf a_1)|\mathbf s_0, \mathbf a_0\big]\Big] \notag\\
        &=  \mathbb E_{\rho, \boldsymbol{\pi}}\bigg[\nabla \log \pi_j(a_0^j|s_0^j)\Big[\gamma \mathbb E_{s^i_1, a^i_1}[r^i_1 (s^i_1, a^i_1)|\mathbf s_0, \mathbf a_0] + \gamma^2 \mathbb E_{\mathbf s_2, \mathbf a_2}\big[Q^i(\mathbf s_2, \mathbf a_2)|\mathbf s_0, \mathbf a_0 \big]
        \Big]\bigg], \label{eq:the_second_term}
    \end{align}
where in the third equality, we use $\mathbb E_{a_0^j}[\nabla \log \pi_j(a_0^j|s_0^j)r^i_0(s_0^i, a_0^i)]=0$ when $i \neq j$ (since we are assuming local reward structure), which we can write as
\[(1-\gamma)\nabla_{\pi_j}J^i(\boldsymbol{\pi}) = \mathbb E\Big[\nabla \log \pi_j(a_0^j|s_0^j) \mathbb E_{\tau} \big[  r^i_0 1(j\in \{i\}) + \gamma Q^i(\mathbf s_1, \mathbf a_1) | \mathbf s_0, \mathbf a_0\big]  \Big].\]
Next, observe the first term in \eqref{eq:the_second_term} that 
\begin{align*}
    \mathbb E_{s^i_1, a^i_1}[r^i_1(s^i_1, a^i_1)|\mathbf s_0, \mathbf a_0] &= \sum_{s^i_1 \in \mathcal S^i}\sum_{ a^i_1 \in \mathcal A^i}P^i\big(s^i_1|\{s^k_0, a^k_0:k\in \text{Pa}^i(s_0^i)\}\big)\pi_i(a_1^i|s^i_1)r^i_1(s_1^i, a_1^i)
\end{align*}
So given $\mathbf s_0$, if $j \not \in \text{Pa}^i(s_0^i)$, then 
\begin{align*}
    &\mathbb E_{\mathbf a_0\sim \boldsymbol{\pi}}\Big[\nabla \log \pi_j(a_0^j|s_0^j) \mathbb E_{s^i_1, a^i_1}[r_1^i(s^i_1, a^i_1)|\mathbf a_0] \;\Big |\:\mathbf s_0\Big] \\
    &\qquad = \mathbb E_{\mathbf a_0}\Big[\nabla \log \pi_j(a_0^j|s_0^j) \mathbb E_{s^i_1, a^i_1}[r^i_1(s^i_1, a^i_1)|\mathbf a_0^{-j}] \;\Big |\:\mathbf s_0\Big]\\
    &\qquad =0.
\end{align*}
Note that we can interpret the property $j \not \in \text{Pa}^i(s_0^i)$ as: given a segment of state trajectory $\boldsymbol{s}_{0:1}= (\mathbf s_0, \mathbf s_1)$ for an arbitrary $\mathbf s_1 \in \boldsymbol{\mathcal S}$, then the path (defined on this segment of trajectory) $(\mathbf s_0, j) \rightarrow (\mathbf s_1, i)$ is not a path in the dependency graph. If this property holds, then the policy gradient of agent $j$ with respect to the reward of agent $i$ at the current timestep $t=1$ is 0. In other words,
\[(1-\gamma)\nabla_{\pi_j}J^i(\boldsymbol{\pi}) = \mathbb E\Big[\nabla \log \pi_j(a_0^j|s_0^j) \mathbb E_{\tau} \big[r^i_0 1(j\in \{i\})+  \gamma r^i_1 1(j\in \text{Pa}^i(s^i_0)) + \gamma^2Q^i(\mathbf s_2, \mathbf a_2) | \mathbf s_0, \mathbf a_0\big]  \Big].\]

We next  consider the second term in \eqref{eq:the_second_term}
\begin{align*}
     &\mathbb E_{\rho, \boldsymbol{\pi}}\bigg[\nabla \log \pi_j(a_0^j|s_0^j)\Big[\gamma^2 \mathbb E_{\mathbf s_2, \mathbf a_2}\big[Q^i(\mathbf s_2, \mathbf a_2)|\mathbf s_0, \mathbf a_0 \big]
        \Big]\bigg]\\
        &\qquad = \mathbb E_{\rho, \boldsymbol{\pi}}\bigg[\nabla \log \pi_j(a_0^j|s_0^j)\Big[\gamma^2 \mathbb E_{s_2^i, a_2^i}[r^i_2(s_2^i, a_2^i)|\mathbf s_0, \mathbf a_0] +  \gamma^3 \mathbb E_{\mathbf s_3, \mathbf a_3}\big[Q^i(\mathbf s_3, \mathbf a_3)|\mathbf s_0, \mathbf a_0 \big]
        \Big]\bigg].
\end{align*}
Again, we write the reward term in its explicit form
\begin{align*}
    &\mathbb E_{s^i_2, a^i_2}[r^i(s^i_2, a^i_2)|\mathbf s_0, \mathbf a_0] \\
    &= \begin{multlined}[t]
        \sum_{\mathbf s_1 \in \boldsymbol{\mathcal S}} \sum_{\mathbf a_1\in \boldsymbol{\mathcal A}} \sum_{s^i_2 \in \mathcal S^i}\sum_{ a^i_2 \in \mathcal A^i}  
    \bigg[r^i_2(s_2^i, a_2^i) P^i\big(s^i_2|\{s^k_1, a^k_1:k\in \text{Pa}^i(s_1^i)\}\big)\pi_i(a_2^i|s^i_2)\\
    \prod_{n}^N \Big[ P^n \big( s^n_1 |  \{s^m_0, a^m_0:m\in \text{Pa}^n(s_0^n)\} \big)\pi_n(a_1^n|s_1^n) \Big]\bigg]
    \end{multlined}\\
    &= \begin{multlined}[t]
        \sum_{\mathbf s_1 \in \boldsymbol{\mathcal S}} \sum_{\mathbf a_1\in \boldsymbol{\mathcal A}} \sum_{s^i_2 \in \mathcal S^i}\sum_{ a^i_2 \in \mathcal A^i}  
    \Bigg[r^i_2(s_2^i, a_2^i) P^i\big(s^i_2|\{s^k_1, a^k_1:k\in \text{Pa}^i(s_1^i)\}\big)\pi_i(a_2^i|s^i_2)\\
    \prod_{k\in \text{Pa}^i(s_1^i)}^{} \Big[ P^k \big( s^k_1 |  \{s^m_0, a^m_0:m\in \text{Pa}^k(s_0^k)\} \big)\pi_k(a_1^k|s_1^k) \Big]\Bigg],
    \end{multlined}
\end{align*}
the second equality is due to the marginal expectation of agents that are not in $\text{Pa}^i(s_1^i)$.
So given $\mathbf s_0$ and $\mathbf s_1$, if $j \not \in \cup_{k\in \text{Pa}^i(s_1^i)}\text{Pa}^k(s_0^k)$, then
\begin{align*}
    &\mathbb E_{\mathbf a_0\sim \boldsymbol{\pi}}\Big[\nabla \log \pi_j(a_0^j|s_0^j) \mathbb E_{s^i_2, a^i_2}[r(s^i_2, a^i_2)|\mathbf a_0, \mathbf s_1] \;\Big |\:\mathbf s_0\Big] 
    =0,
\end{align*}
we can interpret the property $j \not \in \cup_{k\in \text{Pa}^i(s_1^i)}\text{Pa}^k(s_0^k)$ as: given a segment of trajectory $\boldsymbol{s}_{0:2}= (\mathbf s_0, \mathbf s_1, \mathbf s_2)$ for an arbitrary $\mathbf s_2 \in \boldsymbol{\mathcal S}$, then all the paths (defined on this segment of trajectory) $(\mathbf s_0, j) \rightarrow (\mathbf s_1, \cdot) \rightarrow (\mathbf s_2, i)$ is not a path in the dependency graph $\mathcal G$. If this property holds, then the policy gradient of agent $j$ with respect to the reward of agent $i$ at the current timestep $t=2$ is 0. In other words,
\begin{align*}
    \begin{multlined}
        (1-\gamma)\nabla_{\pi_j}J^i(\boldsymbol{\pi}) = \mathbb E\bigg[\nabla \log \pi_j(a_0^j|s_0^j) \mathbb E_{\tau} \Big[ r^i_0 1(j\in \{i\})+ \gamma r^i_1 1\big(j\in \text{Pa}^i(s^i_0)\big) \\ + \gamma^2 r_2^i 1\Big(j \in \cup_{m\in \text{Pa}^i(s_1^i)} \text{Pa}^m(s_0^m)\Big) 
        + \gamma^3Q^i(\mathbf s_3, \mathbf a_3) | \mathbf s_0, \mathbf a_0\Big]  \bigg].
    \end{multlined}
\end{align*}
Also, notice that $\cup_{k\in \text{Pa}^i(s_1^i)}\text{Pa}^k(s_0^k)$ is the set of 2-hop parents of $s_2^i$ on the realised trajectory $\tau$.

In general, with trajectories of states $\tau =(\mathbf s_0, \mathbf s_1, \mathbf s_2, \dots)$, by repeating the above steps, we can write the policy gradient as
\begin{align*}
    (1-\gamma)\nabla_{\pi_j}J^i(\boldsymbol{\pi}) &= \mathbb E_{\mathbf s_0\sim \rho_{\boldsymbol{\pi}}, \mathbf a_0 \sim \boldsymbol{\pi}}\bigg[\nabla \log \pi_j(a_0^j|s_0^j) \mathbb E_{\tau} \Big[ \sum_{t=0}^\infty\gamma^t r^i_t 1\big(j \in \text{Pa}^i(s_t^i, t, \tau) \big) \big| \boldsymbol{s}_0, \boldsymbol{a}_0 \Big] \bigg],
\end{align*}
where we define $ \text{Pa}^i(s_t^i, k, \tau)$ the $k$-hop parents of $s_t^i$ on the trajectory $\tau$, that is
\begin{align}
    \text{Pa}^i(s_t^i, k, \tau)&:= \text{Pa}^i(s^i_{t}, k-1, \tau)\cup\bigg(\bigcup_{m\in \text{Pa}^i(s^i_{t}, k-1, \tau)}\text{Pa}^m\big(s^m_{t-k}\big)\bigg) \notag\\
    &\; = \bigcup_{m\in \text{Pa}^i(s^i_{t}, k-1, \tau)}\text{Pa}^m\big(s^m_{t-k}\big), \label{eq:k_hop_def}
\end{align}
where the last equality is because $i\in \text{Pa}^i(s^i), \forall s^i\in \mathcal S^i$. As the base case, we define the 0-hop parents of any state as $\text{Pa}^i(s_t^i, 0, \tau) = \{i\}$. 

Finally, we can see that the sets $\{\text{Pa}^i(s_t^i, t, \tau)\}_t$ is a non-decreasing sequence, because \eqref{eq:k_hop_def} can also be written as
\begin{equation}
    \text{Pa}^i(s_t^i, t, \tau) = \bigcup_{n\in \text{Pa}^i(s^i_{t-1})}\text{Pa}^{n}(s^n_{t-1}, t-1, \tau), \label{eq:k_hop_alternative_def}
\end{equation}
and that $i \in \text{Pa}^i(s^i_{t-1})$.
As a result, if there exists a timestep $t' > 0$ where $j\in \text{Pa}^i(s_{t'}^i, t', \tau)$, then all subsequent timespteps $t{''} \geq t'$, $j\in \text{Pa}^i(s_{t''}^i, t'', \tau)$. We can let $t'$ be the smallest of such timesteps (in which case $t'=T_{ji}(\tau, s_0^j)$), then
\[(1-\gamma)\nabla_{\pi_j}J^i(\boldsymbol{\pi}) =\mathbb E_{\rho_{\boldsymbol{\pi}}, \boldsymbol{\pi}}\left[\nabla \log \pi_j(a_0^j|s_0^j)\mathbb E_{\tau}\bigg [ \sum_{k \geq t'}\gamma^{k}r^i(s^i_{k}, a^i_{k})\bigg| \mathbf s_0, \mathbf a_0\bigg]\right].\]
If there does not exist such a timestep, then we can simply let $\nabla_{\pi_j}J^i(\boldsymbol{\pi})= 0$ on this trajectory, which is equivalent to $t'=\infty$.
\end{proof}

%% file: mutual_infomation_bound.tex
\begin{proof}[Derivation of the inequality \eqref{eq:mutual_info_bound}]
Given the true graph $\mathcal G$ and an approximated graph $\mathcal G'$.
The transition kernel $\mathbf P$ that is restricted to the graph structure of $\mathcal G'$ can be defined as the marginal distribution:
\[P^i_{\mathcal G'}(s^i| \mathbf s, \mathbf a) =   P^i\big(s^i{'} \big| s^i, a^k\,:\,k\in \text{Pa}^i_{\mathcal G'}(s^i)\big)\quad \forall i\in [N].\]
Now we consider too possible scenarios with the mismatch between $\mathcal G$ and $\mathcal G'$:
\begin{enumerate}
    \item $\mathcal G'$ contains an edge that does not exist in $\mathcal G$. In other word, $\exists \mathbf s \in \boldsymbol{\mathcal S}, i, j \in [N]$ such that $j\in \text{Pa}^i_{\mathcal G'} (\mathbf s)$ and $j\not\in \text{Pa}^i(\mathbf s)$. In this case, since $\mathbf P_{\mathcal G}$ is the true transition kernel $\mathbf P$, conditioning $\mathbf P$ on additional substate-action pairs does not change the next states distributions. As a result, this has no effect on the distribution of $\mathbf P_{\mathcal G}$.
    \item $\mathcal G$ contains an edge that does not exist in $\mathcal G'$. In this case, we are essentially predicting next states with missing information; $\mathbf P_{\mathcal G}$ is the average next states prediction marginalized on these missing edges.
\end{enumerate}
As a result, we can safely ignore the former case and focus on the latter; we let for each state $s^i$, the set $\xi^i( s^i) = \text{Pa}^i( s^i)\setminus \text{Pa}^i_{\mathcal G}(s^i)$.
For simplicity, we assume that $|\xi^i( s^i)| = 1$; $\mathcal G$ differs from $\mathcal G'$ on only a single dependency agent at substate $ s^i$. 
General cases can be easily expanded by a "bridging" argument; by gradually shrinking $\mathcal G$ to $\mathcal G'$ one edge at a time, which we will present later in the proof. 

Let $m( s^i)$ be that missing agent, i.e. $\xi^i(s^i) = \{m(s^i)\}$. We then calculate the total variation of the difference between $ P^i$ and $P^i_{\mathcal G}$ as
\begin{align*}
    &\mathbb E_{\mathbf s, \mathbf a}\|P^i(\cdot|\mathbf s, \mathbf a) - P^i_{\mathcal G'}(\cdot| \mathbf s, \mathbf a)\|_1 \\
    & \begin{multlined}
        \qquad= \mathbb E_{\mathbf s, \mathbf a} \bigg\|P^i\left(\cdot\big|\left\{(s^k, a^k); k\in \text{Pa}_{\mathcal G'}^i( s^i )\right\}\cup\big\{\big(s^{m( s^i)}, a^{m( s^i )}\big)\big\}\right) \\
        - P^i\left(\cdot\big|\left\{(s^k, a^k); k\in \text{Pa}_{\mathcal G'}^i( s^i)\right\}\right)\bigg\|_1
    \end{multlined}\\
    &\qquad = \begin{multlined}[t]
        \mathbb E_{ s^i,  a^i} 
    \mathbb E_{\mathbf (s, a)^{-i, -m(s^i)}}
    \mathbb E_{s^{m(s^i)},  a^{m(s^i)}}  
    \bigg\|P^i\left(\cdot\big|\left\{(s^k, a^k); k\in \text{Pa}_{\mathcal G'}^i(s^i)\right\}\cup\big\{\big(s^{m(s^i)}, a^{m(s^i)} \big)\big\}\right) - \\ 
    P^i\left(\cdot\big|\left\{(s^k, a^k); k\in \text{Pa}_{\mathcal G'}^i(\mathbf s)\right\}\right)\bigg\|_1
    \end{multlined} \\
    &\qquad \leq 
    \begin{multlined}[t]
        \sqrt{2}\mathbb E_{} \mathbb E_{s^{m(s^i)}, a^{m(s^i)}}\Bigg[
    D_{\text{KL}}\bigg[P^i\left(s^i{'}\big|\left\{(s^k, a^k); k\in \text{Pa}_{\mathcal G'}^i(s^i)\right\}\cup\big\{\big(s^{m(s^i)}, a^{m(s^i)}\big)\big\}\right) \\ 
    \bigg\| P^i\left(s^i{'}\big|\left\{(s^k, a^k); k\in \text{Pa}_{\mathcal G'}^i( s^i)\right\}\right) \bigg]^{1/2}
    \Bigg]
    \end{multlined}
    \\
    &\qquad \leq \sqrt{2}
    \mathbb E_{\mathbf s, \mathbf a} \left[
    I(S^i{'}; (S^{m(s^i)}, A^{m(s^i)}))^{1/2}\big|\left\{(s^k, a^k); k\in \text{Pa}_{\mathcal G'}^i(\mathbf s)\right\}\right]\\
    &\qquad \leq \sqrt{2} \sum_j^N \mathbb E_{\mathbf s, \mathbf a} \left[
    I(S^i{'}; (S^{j}, A^{j}))^{1/2}\big|\left\{(s^k, a^k); k\in \text{Pa}_{\mathcal G'}^i(\mathbf s)\right\}\right],
\end{align*}
where the expectations are over the joint state and action distributions (of an arbitrary policy). The first inequality is due to Pinsker's inequality, the second inequality is a combination of Jensen's inequality (with square root) and the mutual information relation to the KL divergence. The last inequality is due to the fact that $m(s^i) \in [N]$.

Now, for a general graph $\mathcal G'$, we construct a sequence of graphs $\mathcal G_1, \mathcal G_2, \dots, \mathcal G_N$, where any two consecutive graphs differ from each other at most one edge at each state $s^i$, i.e. $|(\text{Pa}_{\mathcal G_n}^i\setminus\text{Pa}_{\mathcal G_{n+1}})( s^i)|\leq 1$. By triagular inequality,

\begin{align*}
    \mathbb E\|P^i-P^i_{\mathcal G'}\|_1 &\leq \mathbb E\|P^N-P^i_{\mathcal G'}\|_1 + \dots + \mathbb E \|P^i_{\mathcal G_1}-P^i\|_1\\
    &\leq \sqrt{2} \sum_n^N \mathbb E\bigg[ \sum_j^N I(S^i{'}; (S^{j}, A^{j})))^{1/2} \;\Big|\;\{(S^k, A^k); k \in \text{Pa}^i_{\mathcal G_n}(\mathbf s)\}\bigg].
\end{align*}
\end{proof}

%% file: reward_dependence_graph.tex
\subsection{Reward Dependence Graph}\label{app:reward-graph}

In addition to state transitions, agents may also interact through
their immediate rewards. To capture such interactions, a reward-dependence graph can be defined that specifies how agents’ actions
influence one another’s instantaneous rewards.
Formally, we allow the reward of agent $i$ to depend on the
state–action pairs of a subset of agents, denoted by
$\text{Pa}^i_\textnormal{rw} \subseteq [N]$. Similar to the state-dependence set of the transition probabilities, the reward function of agent $i$ can be defined in the same way as
\[r^i \left(\{s_k, a_k \; |\; k \in \text{Pa}^i_\textnormal{rw}(s_i) \} \right)\in [0, R_\text{max}].\]
Also, in this section only, we will denote $\text{Pa}^i_\textnormal{st}(\cdot)$ as the state dependence sets to distinguish it from the reward dependence sets explicitly.

\begin{proposition}\label{prop:global_generalized}
    Fix a joint policy $\boldsymbol{\pi}$. 
    Let $i, j \in \mathcal N$.
    Let $\textnormal{Pa}^i_\textnormal{rw}(s_i)$ and $\textnormal{Pa}_\textnormal{st}^i(s_i)$ be the reward- and state-dependence agent set at state $s_i$ of agent $i$, respectively. 
    The policy gradient $\nabla_{\pi_j}\mathcal J^i(\boldsymbol{\pi})$ is given by
    \begin{equation}
        \mathbb E_{\rho_{\boldsymbol{\pi}}, \boldsymbol{\pi}}\left[\nabla \log \pi_j(a_0^j|s_0^j)\mathbb E_{\tau}\bigg [ \sum_{t=0} \gamma^{t} r^i_t  1\Big(j \in  \textnormal{Pa}^i(s_t^i,t, \tau) \Big)\bigg| \mathbf s_0, \mathbf a_0 \bigg]\right], \label{eq:dependency_graph_pg_rw}
    \end{equation}
    where $\textnormal{Pa}^i(s_t^i,k, \tau)$ with $k > 0$ is the $k$-hop parents of $s_t^i$ on the trajectory $\tau$, defined as
    \[\textnormal{Pa}^i(s_t^i, k, \tau)=\bigcup_{m\in \textnormal{Pa}^i(s^i_{t}, k-1, \tau)} \textnormal{Pa}_\textnormal{st}^m\big(s^m_{t-k}\big),\]
    and we define the base-case of 0-hop parents as $\textnormal{Pa}^i(s_t^i, 0, \tau) = \textnormal{Pa}^i_\textnormal{rw}(s_t^i)$. 
    A path is defined according to the state dependence graph $\mathcal G$, similar to Proposition \ref{prop:global}.
\end{proposition}

We note that the difference between this result and the result in \ref{prop:global} is the definition of $k$-hop parents. 
In particular, the $0$-hop parents of an agent are initialized using the reward-dependence set, rather than the state-dependence graph alone. 
Proposition \ref{prop:global_generalized} is strictly more general than proposition \ref{prop:global}.
Indeed, the $0$-hop parent set defined in \ref{app:proof} is a special case of $\textnormal{Pa}^i(s_t^i, 0, \tau)$ under the local reward settings $\textnormal{Pa}^i_\textnormal{rw}(s_t^i) = \{i\}$.
The proof of Proposition~\ref{prop:global_generalized} follows identically to that of Proposition~\ref{prop:global} and is therefore omitted.
Intuitively, the reward dependence set captures the immediate effects and thus contributes directly to the 0-hop parents.
However, longer-term influences arise through the state-dependence graph, as state transitions propagate the effects of agents’ states and actions into the future.
Consequently, the parent sets are initialized by the reward-dependence sets and then expand recursively according to the state-dependence sets.

%% file: sample_analysis_proof.tex
\section{Proof of Theorem \ref{theom:sample_complexity}}\label{sec:main_proof}
In this section, we let $\boldsymbol{\theta}=(\theta_1, \theta_2, \dots, \theta_N)$ with each $\theta_i$ the parameters of agent $i$'s policy. Additionally, we sometimes use $\pi_j$ and $\pi_{\theta_j}$ interchangeably. We assume the following standard assumption on Lipchitzness and smoothness to hold for each agent.
\begin{assumption}[single-agent E-LS \cite{yuan2022general}]\label{assm:e-ls}
    There exists constants $G, F>0$ such that for every state $\mathbf s$ and every agent $i$,
    \[\mathbb E_{a^i\sim\pi_i} \left[\| \nabla_{\theta_i}\log \pi_i(a^i|s^i) \|^2  \right]\leq G^2\]
    \[\mathbb E_{a^i\sim\pi_i} \left[\| \nabla_{\theta_i}^2\log \pi_i(a^i|s^i) \|  \right]\leq F\]
\end{assumption}
Before presenting the proof, we want to note a critical detail that we glossed over in the main text. That is, at the timestep $t$, we do not have enough information to determine $T_{ji}(s_t^j)$. In simple terms, the meeting time (or equivalently the length of the path defined in \ref{def:path}) depends on the sequence of future state trajectories that are yet to be observed at $t$. In formal terms, $T_{ji}$ is not a stopping time. This will not be a problem if we sample the whole infinite trajectory $\tau_{0:\infty}$ and then consider the meeting time in hindsight. In a practical setting, we sample trajectories at a truncated $H$ length. As a result, the meeting time that we consider in this setting should only consider paths that are determined before time $H$; any interaction beyond the time horizon $H$ will be dropped, that is, we consider such meeting times at $H$ to be $\infty$ (or equivalent to saying they are unreachable), even though they can meet in a finite future. Formally, we define 
\begin{equation}
    T_{ji}^H(s_t^i, \mathbf s_{0:H}, \mathcal G)=\min\left(\{H \geq t'\geq t: j \in \text{Pa}^i(s^i_{t'}, t'-t, \mathbf s_{0:t'})  \} \cup \{\infty\}\right). \label{eq:truncated_graph_gradient_estimate}
\end{equation}
Which is totally well-defined after observing the first $H$ timesteps. 
Also, note that $T_{ji}^H$ is an upper bound of $T_{ji}$, $T_{ji}^H(s_t^j, \tau)\geq T_{ji}(s_t^j, \tau), \forall t \leq H$.
Then the $m$ trajectories estimate of the dependent graph policy gradient can be defined based on this new meeting timestep,
\[\hat{\nabla}^m_{\theta_j}J^i(\boldsymbol{\theta}, \mathcal G) = \frac{1}{m}\sum_{l=1}^m\sum_{t=0}^{H}\gamma^t \nabla_{\theta_j}\log \pi_j(a^j_{t, l}|s^j_{t, l}) \sum_{k=t + T_{ji}^H(s_{t, l}^j)}^H\gamma^k r_{k, l}^i.\]

Define $K_{ji}(n, \mathbf s_{0:n}):=\max\{k\leq n: j \in \text{Pa}^i(s_n^i, n-k, \mathbf s_{0:n}) \}\leq n$ to be the last timestep at which the action has a consequence to timestep $n$, and $g_j^i(\tau)$ and $g_j^i(\tau, \mathcal G)$ to be the single-trajectory estimate of traditional policy gradient and dependence graph policy gradient, respectively. For convenience, we also write $K_{ji}(n)$ when the (state) trajectory is obvious from the context. We then quantify the single-trajectory estimated gradient norms as
\begin{align*}
    \mathbb E_\tau \left[\| g^i_j(\tau, \mathcal G)\|^2\right] &= \mathbb E \left[ \left\| \sum_{t=0}^{H-1} \nabla \log \pi_j(a^i_t|s^i_t) \sum_{k=t+T_{ji}^H}^{H-1}\gamma^k r_k^i \right\|^2 \right] \notag\\
   &= \mathbb E \left[ \left\| \sum_{t=0}^{H-1} \gamma^t r_t^i \sum_{k=0}^{K_{ji}(t)} \nabla \log \pi_j(a^i_k|s^i_k) \right\|^2 \right] \notag\\
    &= \mathbb E \left[ \left\| \sum_{t=0}^{H-1} \gamma^{t/2}r^i_t \gamma^{t/2}\sum_{k=0}^{K_{ji}(t)}\nabla \log \pi_j(a^i_k|s^i_k) \right\|^2 \right] \notag\\
    &\leq \mathbb E \left[ \left(\sum_{k'=0}^{H-1}\gamma^{k'} (r_{k'}^i)^2  \right) \left( \sum_{t=0}^{H-1} \gamma^{t} \left\|\sum_{k=0}^{K_{ji}(t)}\nabla \log \pi_j(a^i_k|s^i_k) \right\|^2 \right)\right] \notag\\
    &\leq \frac{R_{\max}^2}{(1-\gamma)}
    \mathbb E\left[ \left( \sum_{t=0}^{H-1} \gamma^{t} \left\|\sum_{k=0}^{K_{ji}(t)}\nabla \log \pi_j(a^i_k|s^i_k) \right\|^2 \right)\right] \notag\\
    &= \frac{R_{\max}^2}{(1-\gamma)} 
    \sum_{t=0}^{H-1} \gamma^{t}\mathbb E\left[ \sum_{k=0}^{K_{ji}(t)} \left\|\nabla \log \pi_j(a^i_k|s^i_k) \right\|^2 \right] \tag{Lemma \ref{lemma:num2}}\\
    &\leq \frac{R_{\max}^2  G^2}{(1-\gamma)}
    \sum_{t=0}^{H-1} \gamma^{t}\mathbb E_\tau[K(t, \tau)]
\end{align*}
We can bound the final term as
\begin{align*}
    \sum_{t=0}^{H-1}\gamma^t K_{ji}(t)&\leq\sum_{t=0}^\infty\gamma^t\sum_{k=0}^t 1(k \leq K_{ji}(t))\\
    &= \sum_{k=0}^\infty \sum_{t={k+T_{ji}(\tau, s_k^j)}}^\infty \gamma^t\\
    &= \sum_{k=0}^{\infty}\frac{\gamma^{k+T_{ji}(\tau, s_k^j)}}{1-\gamma}
\end{align*}
Note that we can now switch back to the meeting time in \eqref{def:meetingtime} when the time horizon $H=\infty$. Taking the expectation
\begin{align}
    \mathbb E\left[ \sum_{t=0}^{H-1}\gamma^t    K_{ji}(t)\right]&\leq \frac{1}{1-\gamma}\sum_{k=0}^\infty \gamma^k \mathbb E_\tau \left[\gamma^{T_{ji}(\tau, s_k^j)}\right]\\
    &\leq \frac{\Gamma_{ji}}{(1-\gamma)^2} \label{eq:boundK}
\end{align}
where we define $\Gamma_{j}^i:=\sup_{k, \boldsymbol{\pi}}\mathbb E_\tau \left[ \gamma^{T_{ji}(\tau, s_k^j)} \right]\leq 1$. Substitute everything back, and we get
\[\mathbb E_\tau \left[\| g_{j}^i(\tau, \mathcal G)\|^2\right] \leq \frac{R_{\max}^2  G^2 \Gamma_{j}^i}{(1-\gamma)^3}\]

Finally, summing over all the cross gradient terms,  
\begin{align*}
    \mathbb E_\tau \left[ \left\| g_j(\tau, \mathcal G) \right\|^2 \right]=
    \mathbb E_\tau \left[ \left\| \sum_i^N g_{j}^i(\tau, \mathcal G) \right\|^2 \right] &\leq \sum_i^N \mathbb E_\tau \left[ \left\|  g_{j}^i(\tau, \mathcal G) \right\|^2 \right]\\
    &\leq \frac{R_{\max}^2G^2}{(1-\gamma)^3}\sum_i^N \Gamma_{j}^i.
\end{align*}
Then we have a bound on the stacked gradient estimate variance 
\begin{equation}
    \mathbb E_\tau [\|g(\tau, \mathcal G)\|^2]:= \mathbb E\tau \left\|\begin{bmatrix}
    g_1(\tau, \mathcal G)\\
    \vdots\\
    g_N(\tau, \mathcal G) 
\end{bmatrix}\right \|^2 = \sum_{j=1}^N\mathbb E_\tau [\| g_j(\tau, \mathcal G)\|^2] \leq \frac{R_{\max }^2 G^2}{(1-\gamma)^3}\sum_{j=1}^N \Gamma_j.\label{eq:est_joint_grad_var}
\end{equation}
where we denote $\Gamma_j=\sum_i^N \Gamma^i_j$. Since $\hat{\nabla}_{\pi_j}J(\boldsymbol{\pi}, \mathcal G)$ and $\hat{\nabla}_{\pi_j}J(\boldsymbol{\pi})$ have the same expectation equals $\nabla_{\pi_j}J_H(\boldsymbol{\pi})$, the following results follow from \cite{yuan2022general}, lemma 4.2, and \eqref{eq:est_joint_grad_var},
\begin{align}
    \mathbb E_\tau \left[ \|\hat{\nabla}_{\boldsymbol{\theta}}J_{}(\boldsymbol{\pi}, \mathcal G)\|^2 \right] &\leq \left( 1-\frac{1}{m} \right)\|\nabla_{\boldsymbol{\theta}}J_H(\boldsymbol{\pi})\| + \frac{G^2R_{\max}^2}{m(1-\gamma)^3}\sum_j^N \Gamma_{j}\label{eq:var_graph_grad}
\end{align} 
\begin{equation}
    \mathbb E_\tau \left[ \|\hat{\nabla}_{\boldsymbol{\theta}}J_{}(\boldsymbol{\pi})\|^2 \right] \leq \left( 1-\frac{1}{m} \right)\|\nabla_{\boldsymbol{\theta}}J_H(\boldsymbol{\pi})\| + \frac{N^2G^2R_{\max}^2}{m(1-\gamma)^3}\label{eq:var_gradient}
\end{equation}
where $\hat{\nabla}_{\pi_j}J(\boldsymbol{\pi}, \mathcal G)=\frac{1}{m}\sum g_j(\tau, \mathcal G)$ and $\hat{\nabla}_{\pi_j}J(\boldsymbol{\pi})=\frac{1}{m}\sum g_j(\tau)$, respectively. 
Then we have the following properties
\begin{align*}
    \mathbb E_{\boldsymbol{a}\sim\boldsymbol{\pi}} \left[\| \nabla_{\boldsymbol{\theta}}\log \boldsymbol{\pi}(\boldsymbol{a}|\mathbf{s}) \|^2  \right]&= \mathbb E_{\boldsymbol{a}\sim\boldsymbol{\pi}} \left[\left\| \nabla_{\boldsymbol{\theta}}\bigg(\sum_i^N\log {\pi_i}({a^i}|{s^i})\bigg) \right\|^2  \right]\\
    &= \sum_i^N \mathbb E_{a^i\sim\pi_i} \left[\| \nabla_{\theta_i}\log \pi_i(a^i|s^i) \|^2  \right] \leq NG^2
\end{align*}
where the first equality is due to the independent structure of each agent's policy, and the second equality is due to the stacked structure of $\boldsymbol{\theta}$. Furthermore, since $\nabla^2_{\boldsymbol{\theta}}\log \boldsymbol{\pi}(\boldsymbol{a}|\mathbf{s})$ is a block diagonal matrix, its norm equals its largest diagonal block spectral norm,
\begin{align*}
\mathbb E_{\boldsymbol{a}\sim\boldsymbol{\pi}} \left[\| \nabla^2_{\boldsymbol{\theta}}\log \boldsymbol{\pi}(\boldsymbol{a}|\mathbf{s}) \|  \right]&=\mathbb E_{\boldsymbol{a}\sim\boldsymbol{\pi}} \left[ \max_i \|\nabla_{\theta_i}^2 \log \pi_i(a^i|s^i)\| \right] \\
    &\leq\sum_i^N \mathbb E_{a^i\sim\pi_i} \left[\| \nabla^2_{\theta_i}\log \pi_i(a^i|s^i) \|  \right] \leq NF
\end{align*}
This suggests that single-agent E-LS (\ref{assm:e-ls}) implies joint E-LS. 
However, we can additionally utilise the dependence graph structure of the Networked MDP to obtain a tighter bound on the smoothness coefficient of the objective $J(\boldsymbol{\theta})$.
\begin{lemma}
    Under assumption \ref{assm:e-ls}, $J(\boldsymbol{\theta})$ is $L$-smooth, namely $\|\nabla^2_{\boldsymbol{\theta}}J(\boldsymbol{\theta}) \|\leq L$, with
    \[L_i=\frac{R_{\max}\max_{j\leq N}\sum_i^N\Gamma_{j}^i}{(1-\gamma)^2}(G^2 + F).\]
\end{lemma}
\begin{proof} 
Denote $g_{ji}(\tau, \mathcal G)= \sum_{t=0}^{\infty} \gamma^t r_t^i \left( 
        \sum_{k=0}^{K_{ji}(t)}\nabla_{\theta_j}
        \log \pi_{\theta_j}(a_k^j|s_k^j) 
        \right)$, $\psi_j(\tau)=\sum_{t=0}^\infty\nabla_{\theta_j}\log\pi _{\theta_j}(a_t^j|s_t^j)$ and $h_{ji}(\tau, \mathcal G)=\sum_{t=0}^{\infty} \gamma^t r_t^i \left( 
        \sum_{k=0}^{K_{ji}(t)}\nabla^2_{\theta_j}
        \log \pi_{\theta_j}(a_k^j|s_k^j) 
        \right)$, then
    \begin{align*}
        \nabla^2_{\boldsymbol{\theta}} J^i(\boldsymbol{\theta}) &= \nabla _{\boldsymbol{\theta}}\mathbb E_{\tau} \left(\begin{bmatrix}
            g_{1i}(\tau, \mathcal G)\\\vdots \\g_{Ni}(\tau, \mathcal G)
        \end{bmatrix} \right)\\
        &= \begin{multlined}[t]
            \underbrace{\mathbb E_\tau \left[ \begin{bmatrix}
                \psi_1(\tau)\\
                \vdots\\
                \psi_N(\tau)
            \end{bmatrix}
             \begin{bmatrix}
            g_{1i}(\tau, \mathcal G)\\\vdots \\g_{Ni}(\tau, \mathcal G)
        \end{bmatrix}^\top \right]}_{A^i}
        + \underbrace{\mathbb E_\tau \begin{bmatrix}
            h_{1i}(\tau, \mathcal G)\\
        &\ddots\\
        &&
        h_{Ni}(\tau, \mathcal G)
        \end{bmatrix}}_{B^i}
        \end{multlined}
    \end{align*}
    Then the Hessian of the total reward has the form
    \begin{align*}
        \nabla^2_{\boldsymbol{\theta}}J(\boldsymbol{\theta})=\sum_i A^i +\sum_i B^i
    \end{align*}
    We can bound $\|\sum_iB^i\|$ by
    \begin{align*}
        \left\|\sum_i B^i\right\| &= \left\|\begin{bmatrix}
            \sum_i \mathbb E_\tau h_{1i}(\tau, \mathcal G)\\
        &\ddots\\
        &&
        \sum_i \mathbb E_\tau h_{Ni}(\tau, \mathcal G)
        \end{bmatrix}\right\|\\
        &=\max_{j\leq N} \left\| \sum_i^N \mathbb E_\tau h_{ji}(\tau, \mathcal G)\right\|\\
        &\leq \max_{j\leq N}\sum_i^N\mathbb E_{\tau}\left[ \|h_{ji}(\tau, \mathcal G)\| \right]\\
        &= \max_{j\leq N}\sum_i^N\mathbb E_{\tau} \left[\left\| \sum_{t=0}^{\infty} \gamma^t r_t^i 
        \sum_{k=0}^{K_{ji}(t)}\nabla^2_{\theta_j}
        \log \pi_{\theta_j}(a_k^j|s_k^j) 
        \right\| \right]\\
        &\leq \max_{j\leq N}\sum_i^N \mathbb E_{\tau} \left[ \sum_{t=0}^{\infty} \gamma^t r_t^i 
        \sum_{k=0}^{K_{ji}(t)}\left\|\nabla^2_{\theta_j}
        \log \pi_{\theta_j}(a_k^j|s_k^j) 
        \right\| \right]\\
        &\leq R_{\max}\max_{j\leq N} \sum_i^N \sum_{t=0}^{\infty}\gamma^t\mathbb E_{\tau} \left[ 
        \sum_{k=0}^{K_{ji}(t)}\left\|\nabla^2_{\theta_j}
        \log \pi_{\theta_j}(a_k^j|s_k^j) 
        \right\| \right]\\
        &\leq R_{\max}F\max_{j\leq N} \sum_i^N\sum_{t=0}^\infty \gamma^t \mathbb E_\tau [K_{ji}(t)]\\
        &\leq \frac{R_{\max}F \max_{j\leq N}\sum_i^N \Gamma^i_j}{(1-\gamma)^2}
    \end{align*}
    where the final inequality follows similarly to \eqref{eq:boundK}, the second inequality is due to the triangular inequality, and the penultimate inequality is due to the assumption \ref{assm:e-ls} that holds for each state and by lemma \ref{lemma:num2}. Now we need to bound $A$ norm, first observe that $A$ is a block diagonal matrix due to lemma \ref{lemma:block_diag},
    \begin{align*}
        A&=\mathbb E_\tau\begin{bmatrix}
            \sum_i^N g_{1i}(\tau)\psi_1(\tau)^\top \\
            &\ddots\\
            &&
            \sum_i^N g_{Ni}(\tau)\psi_N(\tau)^\top
        \end{bmatrix}
    \end{align*}
    As a result, we can bound the norm of $A$ by
    \begin{align*}
        \|A\|&= \left\| \begin{bmatrix}
            \mathbb E_\tau\sum_i^N g_{1i}(\tau)\psi_1(\tau)^\top \\
            &\ddots\\
            &&
            \mathbb E_\tau\sum_i^N g_{Ni}(\tau)\psi_N(\tau)^\top
        \end{bmatrix}\right\|\\
        &= \max_{j\leq N}\left\|\mathbb E_\tau\sum_i^N g_{ji}(\tau)\psi_j(\tau)^\top\right\|\\
        &\leq \max_{j\leq N}\mathbb E_\tau\sum_i^N \sum_t^\infty \gamma^t r_t^i \left\| \left( \sum_k^{K_{ji}(t)}\nabla_{\theta_j }\log\pi_{\theta_j}(a_k^j|s_k^j) \right)\sum_{t'}^{K_{ji}(t)}\nabla_{\theta_{j} }\log\pi_{\theta_{j}}(a_{t'}^{j}|s_{t'}^{j})^\top \right\|\\
        &\leq R_{\max}\max_{j\leq N} \sum_i^N \sum_t^\infty \gamma^t \mathbb E \left\| \sum_k^{K_{ji}(t)}\nabla_{\theta_j }\log\pi_{\theta_j}(a_k^j|s_k^j) \right\|^2\\
        &\leq R_{\max}G^2 \max_{j\leq N} \sum_i^N \sum_t^\infty \gamma^t \mathbb E [K_{ji}(t)]\\
        &\leq \frac{R_{\max}G^2 \max_{j\leq N} \sum_i^N\Gamma_j^i}{(1-\gamma)^2}.
    \end{align*}
\end{proof}
Without the dependence graph, then $L$ can scale with $N$ due to the scale of the joint E-LS.
Then by Corollary 4.7 in \cite{yuan2022general} applied for the joint parameter vector $\boldsymbol{\theta}$, by choosing stepsize $\eta=\frac{\epsilon^2 m}{2L\nu}$, minibatch size $m$ between 1 and $\frac{2\nu}{\epsilon^2}$, horizon $H$ to $O(\log (N/\epsilon)/\log(1/\gamma))$, and number of iterations $T=\frac{8\delta_0L\nu}{m\epsilon^4}$ where $\delta_0=J^* - J(\boldsymbol{\theta}_0)\leq \frac{NR_{\max}}{1-\gamma}$, we have the sample complexity result of
\begin{align}
    Tm \times H &= O\left(\frac{\delta_0L \nu \log(N/\epsilon)}{\log(1/\gamma)\epsilon^4}\right)\notag \\
    &= \begin{cases}
        {O}\left( \frac{NR_{\max}^4\Gamma\max_{j}\Gamma_j \log (N/\epsilon)}{(1-\gamma)^6\epsilon^4\log(1/\gamma) }\right) & \text{ using dependence graph PG}\\
        {O}\left( \frac{N^3R_{\max}^{4}\max_{j}\Gamma_j \log (N/\epsilon)}{(1-\gamma)^6\epsilon^4\log(1/\gamma) }\right) & \text{ using vanilla PG}
    \end{cases}
\end{align}
to guarantee a convergence to stationary points $\mathbb E\left[ \|\nabla_{\boldsymbol{\theta}^U} J(\boldsymbol{\pi}_{\boldsymbol{\theta}^U})\|^2\right] \leq O(\epsilon^2)$, where $\nu=\frac{G^2 R_{\max}^2\sum_j^N \Gamma_j}{(1-\gamma)^3}$ as in \eqref{eq:var_graph_grad} when using dependence graph policy gradient, and $\nu=\frac{G^2 R_{\max}^2N^2}{(1-\gamma)^3}$ as in \eqref{eq:var_gradient} when using traditional policy gradient.

We consider a simpler version of the gradient domination condition in the literature \citep{agarwal2019reinforcement, yuan2022general} for illustrative purposes.
\begin{assumption}[Gradient domination]\label{assm:gradient_domination}
    There exists a constant $\mu>0$ such that, for all $\boldsymbol{\theta}$,
    \[\|\nabla_{\theta_j} J(\boldsymbol{\theta}) \| \geq \mu \left(\sup_{\pi_j}J(\boldsymbol{\pi}_{\boldsymbol{\theta}}^{-j}, \pi^j) -J(\boldsymbol{\theta}) \right) \quad \forall j\in [N]. \]
\end{assumption}
Then it is easy to see that $O(\epsilon^2)\geq \mathbb E\left[ \|\nabla_{\boldsymbol{\theta}^U} J(\boldsymbol{\pi}_{\boldsymbol{\theta}^U})\|^2\right]=\mathbb E \left[ \sum_j^N \|\nabla_{{\theta}^U_j} J(\boldsymbol{\pi}_{\boldsymbol{\theta}^U}) \|^2 \right] \geq \sum_j^N \left( \mathbb E\left [\|\nabla_{{\theta}^U_j} J(\boldsymbol{\pi}_{\boldsymbol{\theta}^U})\| \right] \right)^2$, so each $\mathbb E \|\nabla_{{\theta}^U_j} J(\boldsymbol{\pi}_{\boldsymbol{\theta}^U})\| \leq O(\epsilon)$. Combine this with assumption \ref{assm:gradient_domination}, we have
\[ \mathbb E[J(\boldsymbol{\boldsymbol{\theta}}^U)]\geq  \sup_{\pi_j}\mathbb E [J(\boldsymbol{\pi}_{\boldsymbol{\theta}^U}^{-j}, \pi^j)]-O(\epsilon), \quad \forall j\in [N].\]

\begin{lemma}\label{lemma:num2}
    For all non-negative integers $t\geq 0$,
    \[\mathbb E_\tau \left[ \left\| \sum_{t=0}^{K_{ji}(t)}\nabla_{\theta_{j}}\log \pi_{\theta_j}(a^j_t|s^j_t)
    \right\|^2\right] = \mathbb E_\tau \left[  \sum_{t=0}^{K_{ji}(t)}\left\|\nabla_{\theta_{j}}\log \pi_{\theta_j}(a^j_t|s^j_t)
    \right\|^2\right] \]
\end{lemma}
\begin{proof}
    The difficulty of this lemma is the dependence of $K_{ji}(t, \tau)$ to the trajectory $\tau$. More specifically, 
    \begin{align*}
 \begin{multlined}[t]
 \mathbb E_\tau \left[ \left\| \sum_{k=0}^{K_{ji}(t)}\nabla_{\theta_{j}}\log \pi_{\theta_j}(a^j_k|s^j_k)
    \right\|^2\right] =
      \mathbb E_\tau \left[  \sum_{k=0}^{K_{ji}(t)}\left\|\nabla_{\theta_{j}}\log \pi_{\theta_j}(a^j_k|s^j_k)
    \right\|^2\right] \\
    +\mathbb E_\tau \left[  \sum_{k=0}^{K_{ji}(t)}\sum_{k'\neq k}^{K_{ji}(t)}\left\langle\nabla_{\theta_{j}}\log \pi_{\theta_j}(a^j_k|s^j_k), \nabla_{\theta_{j}}\log \pi_{\theta_j}(a^j_{k'}|s^j_{k'})
    \right\rangle\right].
    \end{multlined}
    \end{align*}
    Here $K_{ji}(t)$ is measurable with respect to $\mathcal F_t$, while the gradients $\nabla_{\theta_j} \log \pi_{\theta_j}(a_k^j|s_k^j)$ are measurable with respect to $\mathcal F_k \subset \mathcal F_{t}$. In other word, $K_{ji}(t)$ is not a stopping time w.r.t. $\mathcal F_k$. To circumvent this problem, let $S_{0:t}=\sigma(\mathbf{s}_0, \mathbf{s}_1, \dots, \mathbf{s}_t)$ denote the $\sigma$-algebra generated by the state trajectory up to time 
$t$. By the tower property of conditional expectation:
\begin{align}
    \begin{multlined}[b]
        \mathbb E_\tau \left[  \sum_{k=0}^{K_{ji}(t)}\sum_{k'\neq k}^{K_{ji}(t)}\left\langle\nabla_{\theta_{j}}\log \pi_{\theta_j}(a^j_k|s^j_k), \nabla_{\theta_{j}}\log \pi_{\theta_j}(a^j_{k'}|s^j_{k'})
    \right\rangle\right] \\
    =2\mathbb E_{\boldsymbol{s}_{0:t}} \mathbb E_{\boldsymbol{a}_{0:t}} \left[  \sum_{k=0}^{K_{ji}(t)}\sum_{k'> k}^{K_{ji}(t)}\left\langle\nabla_{\theta_{j}}\log \pi_{\theta_j}(a^j_k|s^j_k), \nabla_{\theta_{j}}\log \pi_{\theta_j}(a^j_{k'}|s^j_{k'})
    \right\rangle \bigg| S_{0:t}\right]
    \end{multlined}\label{eq:cross_grad}
\end{align}
Recall that $K_{ji}(t, \boldsymbol{s}_{0:t})=\max\{k\leq t: j \in \text{Pa}^i(s_t^i, t-k, \mathbf s_{0:t}) \}$; the meeting time determined entirely by the state trajectory upto timestep $t$, hence $K_{ji}(t)$ is $S_{0:t}$-measurable; Conditioned on $S_{0:t}$, $K_{ji}(t)$ is a deterministic constant. Then
\begin{align*}
    \eqref{eq:cross_grad} &= 2\mathbb E_{\boldsymbol{s}_{0:t}}\sum_{k=0}^{K_{ji}(t)}\sum_{k'> k}^{K_{ji}(t)} \mathbb E_{\boldsymbol{a}_{0:{k'-1}}} \left[  \left\langle\nabla_{\theta_{j}}\log \pi_{\theta_j}(a^j_k|s^j_k), \underbrace{\mathbb E_{a^j_{k'}}\left[\nabla_{\theta_{j}}\log \pi_{\theta_j}(a^j_{k'}|s^j_{k'})\Big|  s_{k'}^j\right]}_{=0}
    \right\rangle \bigg| S_{0:t}\right]\\
    &=0.
\end{align*}
\end{proof}

\begin{lemma}\label{lemma:block_diag} For any agents $i, j, j'$. Let
    \[(*):=\mathbb E_\tau \left[ \sum_t^\infty \sum_{t'}^\infty \gamma^t r^i_t \left( \sum_k^{K_{ji}(t)}\nabla_{\theta_j }\log\pi_{\theta_j}(a_k^j|s_k^j) \right)\nabla_{\theta_{j'} }\log\pi_{\theta_{j'}}(a_{t'}^{j'}|s_{t'}^{j'})^\top\right],\]
    then $(*)=0$ if $j\neq j'$. And when $j=j'$, then
    \[(*)=\mathbb E_\tau \left[ \sum_t^\infty \gamma^t r^i_t \left( \sum_k^{K_{ji}(t)}\nabla_{\theta_j }\log\pi_{\theta_j}(a_k^j|s_k^j) \right)\left( \sum_{t'}^{K_{ji}(t)}\nabla_{\theta_{j} }\log\pi_{\theta_{j}}(a_{t'}^{j}|s_{t'}^{j})\right)^\top\right].\]
\end{lemma}
\begin{proof}
    We first consider the case when $j\neq j'$ and $j'\neq i$. For any timestep $t$ and $t'$,
    similar to the proof of Lemma \ref{lemma:num2}, we condition on the state trajectory up to a timestep $\max(t, t')$. Define
    \begin{align*}
    (**)&:=\mathbb E\mathbb E \left[ \gamma^t r_t^i \left( \sum_k^{K_{ji}(t)}\nabla_{\theta_j }\log\pi_{\theta_j}(a_k^j|s_k^j) \right)\nabla_{\theta_{j'} }\log\pi_{\theta_{j'}}(a_{t'}^{j'}|s_{t'}^{j'})^\top\bigg|\mathbf s_{0:\max(t, t')} \right] \\
    &=\mathbb E\mathbb E\left[ \gamma^t r_t^i \left( \sum_k^{K_{ji}(t)}\nabla_{\theta_j }\log\pi_{\theta_j}(a_k^j|s_k^j) \right)\underbrace{\mathbb E_{a^{j'}_{t'}}\left[\nabla_{\theta_{j'} }\log\pi_{\theta_{j'}}(a_{t'}^{j'}|s_{t'}^{j'})|s^{j'}_{t'}\right]^\top}_{=0}\bigg|\mathbf s_{0:\max(t, t')} \right]\\
    &=0
    \end{align*}
    where we can integrate over the gradient of agent $j'$ because $j'$ and $j$ are different agents, so their gradients are independent, and the rewards $r^i$ only depend on agent $i\neq j'$.
    
    Next, if $j\neq j'$ and $j'=i$, then 
    \begin{align*}
        (**)&=\mathbb E\mathbb E \left[ \gamma^t r_t^i \left( \sum_k^{K_{ji}(t)}\nabla_{\theta_j }\log\pi_{\theta_j}(a_k^j|s_k^j) \right)\nabla_{\theta_{i} }\log\pi_{\theta_{i}}(a_{t'}^{i}|s_{t'}^{i})^\top\bigg|\mathbf s_{0:\max(t, t')} \right]\\
        &=\mathbb E \mathbb E \left[ \gamma^t r_t^i \left( \sum_k^{K_{ji}(t)}\underbrace{\mathbb E_{a_k^j}\Big[\nabla_{\theta_j }\log\pi_{\theta_j}(a_k^j|s_k^j) |s_k^j\Big]}_{=0}\right)\nabla_{\theta_{i} }\log\pi_{\theta_{i}}(a_{t'}^{i}|s_{t'}^{i})^\top\bigg|\mathbf s_{0:\max(t, t')} \right]\\
        &=0,
    \end{align*}
    where we can integrate over all $a^j_k$ because given $s_t^i$, $r^i_t(s_t^i, a_t^i)$ is independent from all $a^j_k$ with $j\neq i$.
    Combining the two above results, and summing over all timesteps $t$ and $t'$, then we have when $j\neq j'$,
    \begin{align*}
    (*)= \sum_{t}^\infty\sum_{t'}^\infty(**)=0
    \end{align*}

    Finally, we consider the diagonal case when $j=j'$, then
    \begin{align*}
        (*)&=\mathbb E \left[ \sum_t^\infty \gamma^t r_t^i \left( \sum_k^{K_{ji}(t)}\nabla_{\theta_j }\log\pi_{\theta_j}(a_k^j|s_k^j) \right)\sum_{t'}^\infty\nabla_{\theta_{j} }\log\pi_{\theta_{j}}(a_{t'}^{j}|s_{t'}^{j})^\top\right]\\
        &=\mathbb E \left[ \sum_t^\infty \gamma^t r_t^i \left( \sum_k^{K_{ji}(t)}\nabla_{\theta_j }\log\pi_{\theta_j}(a_k^j|s_k^j) \right)\sum_{t'}^t\nabla_{\theta_{j} }\log\pi_{\theta_{j}}(a_{t'}^{j}|s_{t'}^{j})^\top\right]\\
        &= \mathbb E \left[ \sum_t^\infty \gamma^t r_t^i \left( \sum_k^{K_{ji}(t)}\nabla_{\theta_j }\log\pi_{\theta_j}(a_k^j|s_k^j) \right)\sum_{t'}^{K_{ji(t)}}\nabla_{\theta_{j} }\log\pi_{\theta_{j}}(a_{t'}^{j}|s_{t'}^{j})^\top \right]
    \end{align*}
where the last equality is because from proposition \ref{prop:global},
\begin{align*}
    0 &= \nabla_{\theta_j} J^i(\boldsymbol{\theta}) - \nabla_{\theta_j} J^i(\boldsymbol{\theta}, \mathcal G) \\
    &= \begin{multlined}[t]
        \mathbb E_\tau \left[ \sum_t^\infty \gamma^t r_t^i \sum_{t'}^t\nabla_{\theta_j}\log \pi_{\theta_j}(a_{t'}^j|s_{t'}^j) \right] 
        - \mathbb E_\tau \left[ \sum_t^\infty \gamma^t r_t^i \sum_{t'}^{K_{ji}(t)}\nabla_{\theta_j}\log \pi_{\theta_j}(a_{t'}^j|s_{t'}^j) \right]
    \end{multlined}\\
    &= \mathbb E_\tau \left[ \sum_t^\infty \gamma^t r_t^i \sum_{{t'}=T_{ji}(t)+1 }^t\nabla_{\theta_j}\log \pi_{\theta_j}(a_{t'}^j|s_{t'}^j) \right],
\end{align*}
\end{proof}

%% file: env_spec.tex
\section{Environment and reward specifications}\label{appx:env}
\textbf{Star-Spread MPE} is an environment that is inspired by the spread MPE benchmark \citep{lowe2017multi}. There are $N$ agents and $N$ landmarks. Agent $0$ is the \emph{hub};
agents $1,\dots,N{-}1$ are \emph{leaves}. Landmarks are placed on a
circle of radius $0.7$ at evenly spaced angles.
Each agent starts at a uniformly random position in $[-1,1]^2$ and
takes one of five discrete actions (no-op, up, down, left, right) per
step. 
The episode horizon is $T_{\text{ep}}=50$.

Denote by $p_i\in\R^2$ the position of agent $i$ and by $\ell_j$ the
$j$-th landmark position. Let
$d_{ij} := \|p_i - \ell_j\|_2$ be the agent-to-landmark distance.
The per-agent reward at each step is
\begin{equation}
\label{eq:reward-hub}
  r_0^{\text{hub}} \;=\; -\sum_{j=1}^{N}\min_{i}\,d_{ij} + \xi,
  \qquad
  r_i^{\text{leaf}} \;=\; -\min_j\,d_{ij} + \xi
  \quad\text{for } i=1,\dots,N{-}1,
\end{equation}
where $\xi\sim\mathcal N(0, 1)$ are sampled independently from standard Gaussian noise to increase the difficulty of the credit assignment problem. 
The hub is charged the {team coverage loss} (the sum of
each landmark's distance to its nearest agent), while each leaf is only charged its own distance to the closest landmark.
In the local-reward setting, the per-agent rewards in
Eq.~\eqref{eq:reward-hub} are delivered individually;
in the global-reward setting, the scalarized
team reward $R_t=\sum_i r_t^i$ is delivered to every agent.

The hub observes $[p_0, \text{all } p_i, \text{all } \ell_j]\in\R^{4N+2}$
(centralized observation), whereas each leaf observes
$[p_i, \text{all } \ell_j]\in\R^{2N+2}$ (its own position and the landmark map).
The state-dependence graph used as the \emph{oracle} graph is
\begin{equation}
\label{eq:oracle-graph}
  G = \max(I_N, \mathbf{1}_{}e_0^\top),
  \qquad \text{equivalently, } G[i,j]=\1\{i=j\}\vee\1\{j=0\},
\end{equation}
i.e.\ column~$0$ and the diagonal are all ones.
Semantically, the hub's next state depends on all other agents' states and actions, and
every agent's next state depends on its own current state. No
leaf--leaf edges are present.

\textbf{Level-Based Foraging} 
\cite{christianos2020shared} is a suite of several gridworld-based environments, where the agents navigate to collect random spawned food across the map. Both agents and food have a level each, and the food is only collected if the sum of the levels of agents collecting it is larger than or equal to the level of the food. Rewards are assigned to agents upon collection based on their level of contribution to the collecting process.

Additionally, we consider a variation of the LBF benchmark.
In this modified version, we introduce a more challenging reward structure that makes it harder for cooperative behavior among agents to emerge. 
Specifically, while the original LBF environment rewards each agent individually based on their level when collecting food, we adopt a winner-takes-all reward scheme: only the agent with the highest level receives a reward. In cases where multiple agents share the highest level, we break ties in a \textit{deterministic} manner.
Furthermore, the food levels are set so that at least two agents must cooperate to collect them (\verb|-coop| setup).
This setup penalizes myopic agents, while having no impact on approaches that rely on a global reward signal.
We run our experiments on 3 scenarios from vanilla LBF and 3 scenarios with winner-take-all reward setups.

\textbf{SMAClite} 
\cite{michalski2023smaclite} is a fully open-source implementation of the popular SMAC benchmark \cite{samvelyan2019starcraft}. SMAClite allows us to modify the environment to extract the agent-specific reward, which SMAC does not support. More specifically, the damage rewards are now attributed directly to the source agents that cause the damage. Other reward signals that cannot be attributed to a single source, for example, winning rewards and enemies' self-damage, are distributed evenly to all agents. 
Upon killing an enemy, the agents are rewarded with a large reward. As a result, while each agent is rewarded for the damage inflicted on the enemies, the local reward scheme can encourage the agents to "steal" kills from others. 
We run our experiments on 6 selected scenarios, ranging from large team coordination, such as 25m\_vs\_30m, to heterogeneous, role-specific scenarios like 3s5z\_vs\_3s6z. 

%% file: mappo_pseudocode.tex
\begin{algorithm}[]
\caption{Multi-Agent PPO (MAPPO and IPPO) with Dependence Graph.\label{algorithm2}}
\begin{algorithmic}[1]

\STATE Initialize shared actor parameters $\theta$, centralized critic parameters $\vartheta$, encoder $E$, single-agent reverse world model $q_{\varphi}$, action prediction model $q_{\phi}$, multi-agent reverse world model $q_{\psi}$

\FOR{iteration $= 1$ to $N_{\text{iter}}$}

    \STATE Initialize empty buffer $\mathcal{D}$

    \FOR{episode $= 1$ to $N_{\text{episodes}}$}
        
        \FOR{$t = 1$ to $T$}
            \FOR{each agent $i = 1,\dots,N$}
                \STATE Sample action $a_t^i \sim \pi_\theta(\cdot \mid o_t^i)$
            \ENDFOR
            
            \STATE Execute joint action $\mathbf{a}_t = (a_t^1,\dots,a_t^N)$
            \STATE Observe rewards $\{r_t^i\}$, next state $s_{t+1}$, next obs $\{o_{t+1}^i\}$, done flag $d_t$
            
            \STATE Store $(s_t, o_t^i, a_t^i, r_t^i, \log p_t^i, d_t, s_{t+1})$ in $\mathcal{D}$ for all $i$
            
        \ENDFOR
    \ENDFOR

    \FOR{each timestep $t$ in $\mathcal{D}$}
    \STATE Initialize Adjacency matrix $\text{A}_t=\text{diag}(1)$
    \FOR{each agent $i$ in $\mathcal N$}
        \STATE Compute $V_t^i = V^i_\vartheta(s_t)$ if MAPPO and $ V^i_\vartheta(o^i_t)$ if IPPO
        \STATE Compute latent state $z^i_t = E(o_t^i)$
        \STATE Compute action distribution entropy $H_1=H(q_{\phi}(z_t^i))$

        \FOR{each agent $j$ in $\mathcal{N}$}
            \STATE Compute latent state $z^j_t = E(o_t^j)$
            \STATE Compute latent next state $z^j_{t+1} = E(o_{t+1}^j)$
        \STATE Compute multi-agent reverse model distribution entropy $H_2 = H(q_{\psi}(z_t^i, z_{t}^j, z_{t+1}^j))$
        \STATE Compute edge $\text{A}_{ij}=1(H_2/H_1 < c)$
        \ENDFOR
    \ENDFOR
    \ENDFOR

    \STATE Compute dependence truncated Advantage $\{A^i_t\}_{t=0}^{T-1}$ in Algorithm \ref{alg:gae} using $\{V_t^i\}_{t=0}^{T-1}$ and $\{\text{A}_t^i\}_{t=0}^{T-1}$
    \FOR{epoch $= 1$ to $K$}
        \FOR{each minibatch $\mathcal{B}$}
            
            \FOR{each sample $(o_t^i, a_t^i, \log p_{\text{old}}^i, A_t^i)$ in $\mathcal{B}$}
                \STATE $\log p^i = \log \pi_\theta(a_t^i \mid o^i_t)$
                \STATE $r^i = \exp(\log p^i - \log p_{\text{old}}^i)$
                \STATE $L_{\text{clip}}^i = \min\left(r^i A^i_t,\ \text{clip}(r^i,1-\epsilon,1+\epsilon)A^i_t\right)$
            \ENDFOR
            \STATE $L_{\text{actor}} = -\frac{1}{|\mathcal{B}|}\sum_i L_{\text{clip}}^i$

            \STATE $L_{\text{critic}} = \frac{1}{|\mathcal{B}|}\sum (V^i_\vartheta(s) - V^i_{\text{target}})^2$ if MAPPO and $\frac{1}{|\mathcal{B}|} \sum (V^i_\vartheta(o^i) - V_{\text{target}}^i)^2$ if IPPO

            \STATE $L_{\text{entropy}} = -\beta \, \mathbb{E}[\mathcal{H}(\pi_\theta(\cdot|o^i))]$

            \STATE $L = L_{\text{actor}} + c_v L_{\text{critic}} + L_{\text{entropy}}$
            \STATE Update $\theta, \vartheta$ using gradients of $L$
            
        \ENDFOR
    \ENDFOR
    \STATE Compute latent states $\{z_t^i\}_{t, i} = \{E(o_t^i)\}$
    \STATE Compute Encoder loss $L_E = \frac{1}{|\mathcal D| |\mathcal N|}\sum_{i, t} \text{CE}(q_{\varphi}(z_t^i, z_{t+1}^i), a_t^i)$
    \STATE Update Encoder $E$, $\varphi$ using gradient of $L_E$
    \STATE Detach gradient in $\{z_t^i\}_{t, i}$    
    \STATE Compute reverse model loss $L_{\text{rv}}=\frac{1}{|\mathcal D| |\mathcal N|^2}\sum_{i, j, t}\left[ \text{CE}(q_{\phi}(z_t^i), a^i_t) + \text{CE}(q_{\psi}(z_t^i, z_t^j, z_{t+1}^j), a^i_t) \right] $
    \STATE Update $\phi$, $\psi$ using gradient of $L_\text{rv}$
\ENDFOR

\end{algorithmic}
\end{algorithm}